\documentclass{article}
\usepackage[utf8]{inputenc} 
\usepackage[T1]{fontenc}    

\usepackage[a4paper, total={6in, 9in}]{geometry}

\usepackage{graphicx}
\usepackage{hyperref}       
\usepackage{url}            
\usepackage{booktabs}       
\usepackage{amsfonts}       
\usepackage{nicefrac}       
\usepackage{microtype}      
\usepackage[round]{natbib}
\usepackage{amsmath,amsthm,amssymb,amsfonts}
\usepackage[dvipsnames,table]{xcolor}         
\usepackage{xspace}
\usepackage{dsfont}
\usepackage{multirow}
\usepackage{makecell}
\usepackage{nicefrac}
\usepackage{enumitem}
\usepackage{fontawesome5}
\definecolor{pearThree}{HTML}{E74C3C}
\definecolor{pearcomp}{HTML}{B97E29}
\definecolor{pearDark}{HTML}{2980B9}
\definecolor{pearDarker}{HTML}{1D2DEC}
\hypersetup{
	colorlinks,
	citecolor=pearDark,
	linkcolor=pearThree,
	breaklinks=true,
	urlcolor=black}
\usepackage[ruled]{algorithm2e}

\newtheorem{theorem}{Theorem}
\newtheorem{lemma}[theorem]{Lemma}

\newtheorem{proposition}[theorem]{Proposition}
\newtheorem{definition}[theorem]{Definition}

\newtheorem{assumption}[theorem]{Assumption}

\theoremstyle{definition}
\newtheorem{remark}{Remark}

\usepackage{array}
\newcommand{\PreserveBackslash}[1]{\let\temp=\\#1\let\\=\temp}
\newcolumntype{C}[1]{>{\PreserveBackslash\centering}p{#1}}
\newcolumntype{R}[1]{>{\PreserveBackslash\raggedleft}p{#1}}
\newcolumntype{L}[1]{>{\PreserveBackslash\raggedright}p{#1}}

\definecolor{LightCyan}{rgb}{0.88,1,1}
\newcommand{\vtheta}{\theta}
\DeclareMathOperator{\spann}{span}
\newcommand{\wt}[1]{\widetilde{#1}}

\newcommand{\wb}[1]{\overline{#1}}
\newcommand{\norm}[1]{\left\|#1\right\|}
\newcommand{\Reals}{\mathbb{R}}
\DeclareMathOperator{\EV}{\mathbb{E}}

\newcommand{\Sspace}{\mathcal{S}}
\newcommand{\Aspace}{\mathcal{A}}
\newcommand{\indi}[1]{\mathds{1}\left\{#1\right\}}
\newcommand{\Prob}{\mathbb{P}}
\DeclareMathOperator{\Imm}{Im}

\newcommand{\transp}{\mathsf{T}}
\newcommand{\de}{\mathrm{d}}

\newcommand{\argmax}{\operatornamewithlimits{argmax}}

\newcommand{\hls}{{\small\textsc{UniSOFT}}\xspace}
\newcommand{\algo}[1]{{\small\textsc{LSVI-LEADER}}\xspace}

\usepackage{authblk}

\title{Reinforcement Learning in Linear MDPs:\\ Constant Regret and Representation Selection}

\author[1]{Matteo Papini}
\affil[1]{Universitat Pompeu Fabra}
\author[2]{Andrea Tirinzoni}
\affil[2]{INRIA Lille}
\author[3]{Aldo Pacchiano\footnote{Work done while at Facebook AI Research.}}
\affil[3]{Microsoft Research}
\author[4]{Marcello Restelli}
\affil[4]{Politecnico di Milano}
\author[5]{Alessandro Lazaric}
\author[5]{Matteo Pirotta}
\affil[5]{Facebook AI Research}
\date{}

\usepackage[toc,page,header]{appendix}
\usepackage{minitoc}

\begin{document}
\doparttoc 
\faketableofcontents 

\maketitle
\begin{abstract}
   We study the role of the representation of state-action value functions in regret minimization in finite-horizon Markov Decision Processes (MDPs) with linear structure. We first derive a necessary condition on the representation, called universally spanning optimal features (\hls{}), to achieve constant regret in any MDP with linear reward function. This result encompasses the well-known settings of low-rank MDPs and, more generally, zero inherent Bellman error (also known as the Bellman closure assumption). We then demonstrate that this condition is also sufficient for these classes of problems by deriving a constant regret bound for two optimistic algorithms (LSVI-UCB and ELEANOR). 
   Finally, we propose an algorithm for representation selection and we prove that it achieves constant regret when one of the given representations, or a suitable combination of them, satisfies the \hls{} condition.
\end{abstract}

\section{Introduction}\label{sec:introduction}
The ability of an agent to learn an informative mapping from complex observations to a succinct representation is one of the essential factors for the success of machine learning in fields such as computer vision, language modeling, and more broadly in deep learning~\citep{Bengio2013representation}.
In supervised learning, it is well understood that a ``good'' representation is one that allows to accurately fit any target function of interest (e.g., correctly classify a set of objects in an image).
In Reinforcement Learning (RL), this concept is more subtle, as it can be applied to different aspects of the problem, such as the optimal value function or the optimal policy.
Furthermore, recent works have shown that realizability (e.g., being able to represent the optimal value function) is not a sufficient condition for solving an RL problem, as the sample complexity using realizable representations is exponential in the worst case~\citep[e.g.,][]{weisz2021exponential}.
As such, a desirable property of a ``good'' representation in RL is to enable learning a near-optimal policy with a polynomial sample complexity (or similarly sublinear regret bound).

Several works have focused on online learning --- considering sample complexity or regret minimization --- and identified sufficient assumptions for efficient learning.
Standard examples are tabular Markov Decision Processes (MDPs)~\citep[e.g.,][]{JakschOA10,AzarMK12,AzarOM17}, low or zero inherent Bellman error~\citep[e.g.,][]{Jin2020linear,Zanette2020low,Zanette2020frequentist,jin2021bellmaneluder} and linear mixture MDPs~\citep[e.g.,][]{Yang2019matrix,Ayoub2020vtr,Zhang2021variance}. While, in these settings, the representation is provided as input to the algorithm, an alternative scenario is to learn such representations. In this case, research has focused either on the problem of online representation selection for regret minimization~\citep[e.g.,][]{OrtnerMR14,OrtnerPLFM19,LeePMKB21} or, more recently, on the sample complexity of online representation learning~\citep[e.g.,][]{DuKJAD019,AgarwalKKS20,Modi2021modelfree}. Refer to App.~\ref{app:related.work} for more details.
While this literature has focused on finding a representation enabling learning a near-optimal policy with sublinear regret or polynomial sample complexity, there may be several of such ``good'' representations with significantly different learning performance and existing approaches are not guaranteed to find the most efficient one.
Intuitively, we would like to find representations that require the minimum level of exploration to solve the task. For example, representations that would allow the algorithm to stop exploring after a finite time and play only optimal actions forever (i.e., achieving constant regret), if they exist.
This aspect of representation learning was recently studied by~\citet{hao2020adaptive,papini2021leveraging} in contextual linear bandits, where they showed that certain representations display non-trivial properties that enable much better learning performance. While it is well-known that properties such as dimensionality and norm of the features have an impact on the learning performance,~\citet{hao2020adaptive,papini2021leveraging} proved that it is possible to achieve constant regret (i.e., not scaling with the number of learning steps) if a certain (necessary and sufficient) condition on the features associated with the optimal actions is satisfied. To the best of our knowledge, the impact of similar properties on RL algorithms and how to find such representations is largely unexplored.

\textbf{Contributions.} In this paper, we investigate the concept of ``good'' representations in the context of regret minimization in finite-horizon MDPs with linear structure. In particular, we consider the settings of zero inherent Bellman error (also referred to as Bellman closure)~\citep{Zanette2020low} and low-rank structure~\citep[e.g.,][]{Jin2020linear}. 
Similarly to the bandit case~\citep{hao2020adaptive,papini2021leveraging}, we study the impact of representations on the learning process.
Our contributions are both fundamental and algorithmic.
\textbf{1)} We provide a necessary condition (called \hls{}) for a representation to enable constant regret in any problem with linear reward parametrization. Notably, this result encompasses MDPs with zero inherent Bellman error, and linear mixture MDPs with linearly parametrized rewards. Intuitively, the condition generalizes a similar condition for linear contextual bandits and it requires that the features observed along trajectories generated by the optimal actions provide information on the whole feature space (see Asm.~\ref{asm:feature.structure}). 
\textbf{2)} We provide the first \emph{constant regret bound} for MDPs for both ELEANOR~\citep{Zanette2020low} and LSVI-UCB~\citep{Jin2020linear} when the \hls{} condition is satisfied. As a consequence, we show that good representations are not only necessary but also sufficient for constant regret in MDPs with zero inherent Bellman error or low-rank assumptions.
\textbf{3)} We develop an algorithm, called \algo{}, for representation selection in low-rank MDPs. We prove that in low-rank MDPs, \algo{} suffers the regret of the best representation without knowing it in advance. Furthermore, \algo{} achieves constant regret even when only a suitable combination of the representations satisfies the \hls{} condition despite none of them being ``good''. This is indeed possible thanks to its ability to select a different representation for each stage, state, and action.

\section{Preliminaries}\label{sec:preliminaries}
We consider a time-inhomogeneous finite-horizon Markov decision process (MDP) $M = \big(\Sspace, \Aspace, H,$ $\{r_h\}_{h=1}^H, \{p_h\}_{h=1}^H, \mu\big)$ where $\Sspace$ is the state space and $\Aspace$ is the action space, $H$ is the length of the episode, $\{r_h\}$ and $\{p_h\}$ are reward functions and state-transition probability measures, and $\mu$ is the initial state distribution.
We denote by $r_h(s,a)$ the expected reward of a pair $(s,a) \in \Sspace \times \Aspace$ at stage $h$.
We assume that $\Sspace$ is a measurable space with a possibly infinite number of elements and $\Aspace$ is a
finite set.
A policy $\pi = (\pi_1,\ldots,\pi_H) \in \Pi$ is a sequence of decision rules $\pi_h : \Sspace \to \Aspace$.
For every $h \in [H] := \{1, \ldots, H\}$ and $(s,a) \in \Sspace\times\Aspace$, we define the value functions of a policy $\pi$ as 
\[
  Q_h^\pi(s,a) = r_h(s,a)+\mathbb{E}_{\pi}\left[ \sum_{i=h+1}^H r_i(s_i,a_i) \right], \qquad V^\pi_h(s,a) = Q_h^{\pi}(s,\pi_h(s)),
\]
where the expectation is over probability measures induced by the policy and the MDP over state-action sequences of length $H-h$. Under certain regularity conditions~\citep[e.g.,][]{bertsekas2004stochastic}, there always exists an optimal policy $\pi^\star$ whose value functions are defined by $V^{\pi^\star}_h(s):=V^\star_h(s)= \sup_{\pi} V^\pi_h(s)$ and $Q^{\pi^\star}_h(s,a):=Q^\star_h(s,a)= \sup_{\pi} Q^\pi_h(s,a)$.
The optimal Bellman equation (and Bellman operator $L_h$) at stage $h \in [H]$ is defined as:
\[
    Q_h^\star(s,a) := L_hQ_{h+1}^\star(s,a) = r_h(s,a) + \mathbb{E}_{s' \sim p_h(s,a)}\left[ \max_{a'} Q^\star_{h+1}(s',a')\right].
\]
The value iteration algorithm (a.k.a.\ backward induction) computes $Q^\star$ or $Q^\pi$ by applying the Bellman equations starting from stage $H$ down to $1$, with $V_{H+1}^\pi(s) = 0$ by definition for any $s$ and $\pi$. The optimal policy is simply the greedy policy w.r.t. $Q^\star$: $\pi^\star_h(s) = \argmax_{a \in \Aspace} Q^\star_h(s,a)$. 

In online learning, the agent interacts with an \emph{unknown} MDP in a sequence of $K$ episodes. At each episode $k$, the agent observes an initial state $s_{1}^k$, it selects a policy $\pi_k$, it collects the samples observed along a trajectory obtained by executing $\pi_k$, it updates the policy, and reiterates over the next episode. We evaluate the performance of a learning agent through the regret: $R(K) := \sum_{k=1}^K V_1^\star(s_{1}^k) - V_1^{\pi_k}(s_{1}^k)$.

\paragraph{Linear Representation.}
When the state space is large or continuous, value functions are often described through a parametric representation. A standard approach is to use linear representations of the state-action function $Q_h(s,a) = \phi_h(s,a)^\transp \theta_h$, where $\phi_h : \mathcal{S} \times \mathcal{A} \to \mathbb{R}^d$ is a time-inhomogeneous feature map and $\theta_h \in \mathbb{R}^d$ is an unknown parameter vector.\footnote{It is possible to extend the setting to different feature dimensions $\{d_h\}_{h\in [H]}$.} 
In this paper, we consider MDPs satisfying Bellman closure (i.e., zero Inherent Bellman Error)~\citep{Zanette2020low} or low-rank assumptions~\citep[e.g.,][]{Yang2019matrix,Jin2020linear}.

\begin{assumption}[Bellman Closure]\label{asm:zero.ibe}
    Define the set of bounded value function  $\mathcal{Q}_h = \{Q_h | \theta_h\in\Theta_h : Q_h(s,a)=\phi_h(s,a)^\transp \theta_h, \forall (s,a)\}$ and the associated parameter space $\Theta_h =\{\theta_h \in \mathbb{R}^d : |\phi_h(s,a)^\transp \theta_h| \leq D\}$.
    An MDP has zero Inherent Bellman Error (IBE) if
    \begin{align*}
       \forall h\in[H],\qquad \sup_{Q_{h+1} \in \mathcal{Q}_{h+1}} \inf_{Q_h \in \mathcal{Q}_h} \|Q_h - L_hQ_{h+1}\|_{\infty} = 0.
    \end{align*}
\end{assumption}
This definition implies that the optimal value function is realizable as $Q^\star_h \in \mathcal{Q}_h$. Furthermore, the function space $\mathcal{Q}$ is closed under the Bellman operator, i.e., for all $Q_{h+1}\in \mathcal{Q}_{h+1}$, $L_hQ_{h+1} \in \mathcal{Q}_h$. Under this assumption, value-iteration-based algorithms are guaranteed to converge to the optimal policy in the limit of samples and iterations~ \citep{MunosS08}. In the context of regret minimization, \citet{Zanette2020low} proposed a model-free algorithm, called ELEANOR, that achieves sublinear regret under the Bellman closure assumption, but at the cost of computational intractability.\footnote{ELEANOR works under the weaker assumption of low IBE. \citet{jin2021bellmaneluder} considered the more general case of low Bellman Eluder dimension. Their algorithm reduces to ELEANOR in the case of low IBE.} The design of a tractable algorithm for regret minimization under low IBE assumption is still an open question in the literature.

\begin{assumption}[Low-Rank MDP]\label{asm:lowrank}
    Let $\Theta_h = \mathbb{R}^d$, then an MDP has low-rank structure if 
    \begin{align*}
        \forall s,a,h,s', \quad r_h(s,a) = \phi_h(s,a)^\transp \theta_h, \quad p_h(s'|s,a) = \phi_h(s,a)^\transp \mu_h(s')
    \end{align*}
    where $\mu_h : \mathcal{S} \to \mathbb{R}^d$. Then, for any policy $\pi \in \Pi$, $\exists \theta_h^\pi \in \Theta_h$ such that $Q^\pi_h(s,a) = \phi_h(s,a)^\transp \theta_h^\pi$.
        We assume 
        $\norm{\theta_h}_2 \leq \sqrt{d}$, $\|\int_{s'} \mu_h(s')v(s')\mathrm{d}s'\|_2\leq \sqrt{d}\|v\|_{\infty}$ and $\norm{\phi_h(s,a)}_2\leq 1$, for any $s,a,h$, and function $v : \mathcal{S} \to \mathbb{R}$. 
\end{assumption}

This assumption is \textit{strictly} stronger than Bellman closure~\citep{Zanette2020low} and it implies the value function of \emph{any} policy is linear in the features. Furthermore, under Asm.~\ref{asm:lowrank} sublinear regret is achievable using, e.g., LSVI-UCB ~\citep{Jin2020linear}, a tractable algorithm for low-rank MDPs. \citet{He2020loglin} have recently established a problem-dependent logarithmic regret bound for LSVI-UCB under a strictly-positive minimum gap. The minimum positive gap provides a natural measure of the difficulty of an MDP.
\begin{assumption}\label{asm:gap}
    The suboptimality gap of taking action $a$ in state $s$ at stage $h$ is defined as:
    \begin{equation}
        \Delta_h(s,a) = V^\star_h(s) - Q^\star_h(s,a).
    \end{equation}
    We assume the minimum positive gap $\Delta_{\min} = \min_{s,a,h}\{\Delta_h(s,a)|\Delta_h(s,a)>0\}$ is well defined and that the optimal action is unique, i.e., $|\argmax_a \{Q^\star_h(s,a)\}| = 1$, for any $s \in \Sspace$, $h \in [H]$.
\end{assumption}
In Tab.~\ref{tab:regret.bounds}, we summarize existing bounds in the two settings.
Another structural assumption that has gained popularity in the literature is the linear-mixture structure~\citep{JiaYSW20,Ayoub2020vtr,Zhou2020nearly}, where the transition function admits a form $p_h(s'|s,a) = \phi_h(s'|s,a)^\transp \theta_h$.
No structural requirement is made on the reward, which is typically assumed to be known. As a consequence, the value function may not be linearly representable. However, the fact the reward is known and that it is possible to directly learn the parameters $\theta_h$ of the transition function allow to achieve sublinear regret (even logarithmic) through model-based algorithms. While in this paper we mostly focus on Asm.~\ref{asm:zero.ibe} and~\ref{asm:lowrank}, in Sect.~\ref{sec:necessary} we show that our condition is necessary for constant regret also for linear-mixture MDPs with unknown linear reward.

\begin{table}[t!]
\centering
\footnotesize
    \renewcommand{\arraystretch}{2.4}
    \begin{tabular}{|c|c|c|c|c|}
    \hline
    \makecell{Algorithm\\{\scriptsize (setting)}}
    & Minimax & \makecell{Problem-Dependent\\ Logarithmic} & \makecell{Constant with \hls\\ \emph{(this work)}}\\
    \hline 
    \hline
    \makecell{ELEANOR\\ {\scriptsize\emph{(Bellman Closure)}}}
    &\makecell{$\wt{O}(\sqrt{d^2H^3T})$\\ \citep{Zanette2020low}}
    &N/A
    &
    \cellcolor[HTML]{F0F8FF}
    \makecell{$\frac{d^2 H^4}{\Delta_{\min}\lambda_{+}^{3/2}} {\log^{1/2}\bigg(\frac{d^2H^5}{\delta\Delta_{\min}^2\lambda_+^{3}}\bigg)}$\\
    (Thm.~\ref{thm:const.ele})}
    \\
    \hline
    \makecell{LSVI-UCB\\ {\scriptsize\emph{(low-rank MDPs)}}}    
    &\makecell{$\wt{O}(\sqrt{d^3H^3T})$\\ \citep{Jin2020linear}}
    &\makecell{$O(\frac{d^3H^5}{\Delta_{\min}} \log^2(T))$\\ \citep{He2020loglin}}
    &
    
    \cellcolor[HTML]{F0F8FF}
    \makecell{\hspace{.172in}$\frac{d^3 H^5}{\Delta_{\min}} \log \bigg(\frac{d^4H^6}{\delta\Delta_{\min}^2\lambda_+^3} \bigg)$\hspace{.172in}
    \\
    (Thm.~\ref{thm:const.lsvi})}
    \\
    \hline\hline
    Lower Bound
    &\makecell{$\Omega(\sqrt{d^2H^2T})$\\ \citep[][Remark 5.8]{Zhou2020nearly}} 
    &\makecell{$\Omega(\frac{dH}{\Delta_{\min}})$\\ \citep{He2020loglin}}
    &N/A\\
    \hline
    \end{tabular}
\vspace{0.04in}
\caption{
    Regret comparisons of ELEANOR and LSVI-UCB. For ELEANOR, we consider the special case of Bellman closure.
}
\label{tab:regret.bounds}
\end{table}

\section{Constant Regret for Linear MDPs}\label{sec:constant}
In this section, we introduce \hls{}, a necessary condition for constant regret in any MDP with linear rewards. We show that this condition is also sufficient in MDPs with Bellman closure.
\begin{assumption}\label{asm:feature.structure}
    A feature map is \hls{} (Universally Spanning Optimal FeaTures) for an MDP if 
    it satisfies 
    Asm.~\ref{asm:zero.ibe} or~\ref{asm:lowrank}, and for all $h\in[H]$ the following holds:
    \begin{align*}
        \mathrm{span}\Big\{\phi_h(s,a) \;|\; \forall (s,a), \;\exists \pi \in \Pi : \rho^\pi_h(s,a) > 0\Big\} = \mathrm{span}\Big\{\phi_h^\star(s) \;|\; \forall s,\; \rho^{\star}_h(s) > 0\Big\},
    \end{align*}
    where $\rho_h^\pi(s)= \mathbb{E}[\indi{s_h=s}|M,\pi]$ is the occupancy measure of a policy $\pi$, $\rho_h^\pi(s,a) = \rho_h(s) \indi{\pi_h(s) = a}$, $\rho^{\star}_h(s) := \rho^{\pi^\star}_h(s) $, and $\phi^\star_h(s) := \phi_h(s, \pi_h^\star(s))$.
\end{assumption}
Intuitively, features that are observed by only playing optimal actions must provide information on the whole space of reachable features at each stage $h$.
We notice that Asm.~\ref{asm:feature.structure} reduces to the {\small\textsc{HLS}}\xspace property for contextual bandits considered by~\citet{hao2020adaptive,papini2021leveraging}. The key difference is that, in RL, the reachability of a state plays a fundamental role. For example, features of states that are not reachable by any policy are irrelevant, while features of optimal actions in states that are not reachable by the optimal policy (i.e., $\phi^\star_h(s)$ in a state with $\rho_h^\star(s) = 0$) do not contribute to the span of optimal features since they can only be reached by acting sub-optimally.
In RL, a related structural assumption to Asm.~\ref{asm:feature.structure} is the ``uniformly excited feature'' assumption used by~\citet[][Asm. A4]{AbbasiYadkori2019politex} for average reward problems.
Their assumption is strictly stronger than ours since it requires that all policies generate an occupancy measure under which the features span all directions uniformly well. Such an assumption can be related to the ergodicity assumption for tabular MDPs, which is known to be restrictive.
Another related quantity is the ``explorability'' coefficient introduced by~\citet{Zanette2020rewardagnostic}.
This term represents how explorative (in the feature space) are the optimal policies of the tasks compatible with the MDP, i.e., considering any possible parameter $\theta_h \in \Theta_h$. This coefficient is important in reward-free exploration where the objective is to learn a near optimal policy for any task, which is revealed only once learning has completed. In our setting, we focus only on the properties of the optimal policy for the single task we aim to solve.

It is interesting to look into Asm.~\ref{asm:feature.structure} from an alternative perspective.
Denote by $0\leq \lambda_{h,1} \leq \ldots \leq \lambda_{h,d}$ the eigenvalues of the matrix $\Lambda_h := \mathbb{E}_{s \sim \rho_h^{\star}} \big[ \phi_h^\star(s) \phi_h^\star(s)^\transp \big]$ and by $\lambda^+_h= \min\{\lambda_{h,i} >0, i\in[d]\}$ the minimum positive eigenvalue.
We notice that when the features are non-redundant (i.e.,  $\{\phi_h(s,a)\}$ spans $\mathbb{R}^d$) and the \hls{} assumption holds, then $\lambda^+_h = \lambda_{h,1} > 0$.
As we will see, the minimum positive eigenvalue $\lambda_{h}^+$ plays a fundamental role in the constant regret bound, together with the minimum gap $\Delta_{\min}$.
We provide examples of \hls{} and Non-\hls{} representations in App.~\ref{app:examples.experiments}, as well as their impact on the learning process.

\subsection{\hls{} is Necessary for Constant Regret}\label{sec:necessary}
The following theorem shows that the \hls{} condition is necessary to achieve constant regret in a large class of MDPs.
\begin{theorem}\label{thm:necessity}
    Let $M$ be any MDP with finite states, arbitrary dynamics $p$, linear rewards (i.e., $r_h(s,a)=\phi_h(s,a)^\transp \theta_h$) with Gaussian $\mathcal{N}(0,1)$ noise, unique optimal policy $\pi^\star$, and where condition \hls{} (Asm.~\ref{asm:feature.structure}) is not satisfied. Let $\mathcal{M}$ be the set of MDPs with same dynamics as $M$ but different reward parameters $\{\theta_h\}_{h\in[H]}$. Then, there exists no algorithm that suffers sub-linear regret in all MDPs in $\mathcal{M}$ while suffering constant regret in $M$.
\end{theorem}
Thm.~\ref{thm:necessity} states that in MDPs with linear reward, the \hls{} condition is \emph{necessary} to achieve constant regret for any ``provably efficient'' algorithm. 
Notably, this result does not put any restriction on the transition model, which can be arbitrary and known. This means that as soon as the reward is linear and unknown to the learning agent, the \hls{} condition is necessary to attain constant regret. 
This result applies to low-rank MDPs, linear-mixture MDPs with unknown linear rewards, and MDPs with Bellman closure (Bellman closure implies linear rewards, see Prop. 2 by~\citet{Zanette2020low}).

\textbf{Proof sketch of Theorem \ref{thm:necessity}.} 
The key intuition behind the proof is that an algorithm achieving a constant regret must select sub-optimal actions only a finite number of times. Nonetheless, in order to learn the optimal policy, all features associated with suboptimal actions should be explored enough. Since \hls{} does not hold, this cannot happen by executing the optimal policy alone and requires selecting suboptimal policies for long enough, thus preventing constant regret.

More formally, we call an algorithm ``provably efficient'' if it suffers sub-linear regret on the given class of MDPs $\mathcal{M}$. Formally, we use the following definition, which is standard to prove problem-dependent lower bounds \citep[e.g.,][]{Simchowitz2019gap,Xu2021finegrained}.
\begin{definition}[$\alpha$-consistency]\label{def:good-alg-prob}
Let $\alpha \in (0,1)$, then an algorithm $\mathsf{A}$ is $\alpha$-consistent on a class of MDPs $\mathcal{M}$ if, for each $M\in\mathcal{M}$  and $K\geq 1$, there exists a constant $c_M$ (independent from $K$) such that $\mathbb{E}_{M}^{\mathsf{A}}\left[R(K) \right] \leq c_M K^\alpha$.\footnote{In practice, all existing ``provably-efficient'' algorithms we are interested in are included in this class and $c_M$ is polynomial in all problem-dependent quantities (e.g., $d$, $H$). For instance, LSVI-UCB and ELEANOR are $\nicefrac{1}{2}$-consistent on the class of low-rank and Bellman-closure MDPs, where they enjoy worst-case $\wt{O}(\sqrt{K})$ regret bounds (with $c_M$ being $O(\sqrt{d^3H^4})$ and $O(\sqrt{d^2H^4})$, respectively).}
\end{definition}

The following lemma is the key result for proving Thm.~\ref{thm:necessity} and it might be of independent interest. It shows that any consistent algorithm must explore sufficiently all relevant directions in the feature space to discriminate any sub-optimal policy from the optimal one. The proof (reported in App.~\ref{app:necessary}) leverages techniques for deriving asymptotic lower bounds for linear contextual bandits \citep[e.g.,][]{lattimore2017end,hao2020adaptive,tirinzoni2020asymptotically}.
\begin{lemma}\label{lemma:bound-psi}
    Let $M,\mathcal{M}$ be as in Thm.~\ref{thm:necessity} and $\mathsf{A}$ be any $\alpha$-consistent algorithm on $\mathcal{M}$. For any $\pi\in\Pi$, denote by $\Psi_h^\pi := \sum_{s,a}\rho_h^\pi(s,a)\phi_h(s,a)$ its expected features at stage $h$ and $\Delta(\pi) := V_1^\star - V_1^\pi$ its sub-optimality gap. Then, for any $\pi\in\Pi$ with $\Delta(\pi) > 0$ and $h\in[H]$,
    \begin{align*}
\limsup_{K\rightarrow\infty} \log(K)\|\Psi_h^\pi - \Psi_h^{\star}\|_{\mathbb{E}_{{M}}^{\mathsf{A}}[\Lambda_h^K]^{-1}}^2 \leq \frac{\Delta(\pi)^2}{2(1-\alpha)},
\end{align*}
where $\Psi_h^\star := \Psi_h^{\pi^\star}$ and $\Lambda_h^K := \sum_{k=1}^K \phi_h(s_h^k,a_h^k)\phi_h(s_h^k,a_h^k)^\transp$.
\end{lemma}
We now proceed by contradiction: suppose that $\mathsf{A}$ suffers constant expected regret on $M$ even though the MDP does not satisfy the \hls{} condition. Then, since $\mathsf{A}$ plays sub-optimal actions only a finite number of times, it is possible to show that, for each $h\in[H]$, there exists a positive constant $\lambda_M > 0$ such that $\mathbb{E}_{{M}}^{\mathsf{A}}[\Lambda_h^K] \preceq \Lambda_h^\star + \lambda_M I$, where $\Lambda_h^\star := K\sum_{s : \rho_h^\star(s) > 0} \phi^\star_h(s)\phi_h^\star(s)^\transp$. Furthermore, since \hls{} does not hold, there exists a stage $h\in[H]$ and a sub-optimal policy $\pi$ (i.e., with $\Delta(\pi) > 0$) such that the vector $\Psi_h^\pi - \Psi_h^\star$ does not belong to $\mathrm{span}\left\{\phi_h^\star(s) | \rho^{\star}_h(s) > 0\right\}$. Then, since such space is exactly the one spanned by all the eigenvectors of $\Lambda_h^\star$ associated with a non-zero eigenvalue, there exists a positive constant $\epsilon > 0$ (independent of $K$) such that $\|\Psi_h^\pi - \Psi_h^{\star}\|_{(\Lambda_h^\star + \lambda_M I)^{-1}}^2 \geq \epsilon^2 / \lambda_M$. 
That is, even if the (positive) eigenvalues of $\Lambda_h^\star$ grow with $K$, the weighted norm of $\Psi_h^\pi - \Psi_h^\star$, which is not in the span of the eigenvectors of such matrix, cannot decrease below a positive constant. Combining these steps with Lem.~\ref{lemma:bound-psi}, we obtain
\begin{align*}
 \frac{\Delta(\pi)^2}{2(1-\alpha)} 
 \geq \limsup_{K\rightarrow\infty} \log(K)\|\Psi_h^\pi - \Psi_h^{\star}\|_{(\Lambda_h^\star + \eta I)^{-1}}^2 \geq \frac{\epsilon^2}{\lambda_M}\limsup_{K\rightarrow\infty} \log(K),
\end{align*}
which is clearly a contradiction. Therefore, $\mathsf{A}$ cannot suffer constant regret in $M$ while suffering sub-linear regret in all other MDPs in $\mathcal{M}$, and our claim follows.

\subsection{\hls{} is Sufficient for Constant Regret}\label{sec:sufficient}
While the \hls condition is necessary for achieving constant regret in a large class of MDPs, in the following, we prove that ELEANOR and LSVI-UCB attain constant regret when the \hls{} assumption holds, thus implying that it is a sufficient condition in MDPs with low-rank and Bellman closure structure.

\begin{theorem}\label{thm:const.ele}
    Consider an MDP and a representation $\{\phi_h\}_{h\in[H]}$ satisfying the Bellman closure (Asm.~\ref{asm:zero.ibe}) and \hls assumptions (Asm.~\ref{asm:feature.structure}). 
    Under Asm.~\ref{asm:gap}, with probability at least $1-3\delta$, ELEANOR\footnote{ELEANOR and LSVI-UCB are defined up to a regularization parameter $\lambda$ that we set to $\lambda=1$.} suffers a \emph{constant} regret 
    \begin{align*}
        R(K) \lesssim H^{3/2}d\sqrt{\wb\tau\log\frac{\wb\tau}{\delta}}
        ,
    \end{align*}
    where $\wb{\tau} = H\wb\kappa$ and $\wb\kappa$ is the last episode ELEANOR suffers a non-zero regret. Furthermore, $\wb\kappa \lesssim \max\Big\{\frac{d^2H^4}{\lambda_+^{2}},\frac{dH^4}{\Delta_{\min}^{2}\lambda_+^{3}} \Big\}$\footnote{Here $\lesssim$ hides logarithmic terms in $\lambda_+,H$, and $d$, but not in $K$.},
 where $\lambda_+ := \min_h \{\lambda_h^+\} > 0$.
\end{theorem}

Alternatively, we can prove the following result for LSVI-UCB.

\begin{theorem}\label{thm:const.lsvi}
    Consider an MDP and a representation $\{\phi_h\}_{h\in[H]}$ satisfying the low-rank (Asm.~\ref{asm:lowrank}) and \hls{} assumptions (Asm.~\ref{asm:feature.structure}). 
    Under Asm.~\ref{asm:gap}, with probability $1-3\delta$, LSVI-UCB suffers a \emph{constant} regret
    \begin{align*}
        R(K)\lesssim \frac{d^3 H^5}{\Delta_{\min}} \log \big(dH^2 \wb\kappa / \delta\big),
    \end{align*}
    where $\wb\kappa$ is the last episode LSVI-UCB suffers a non-zero regret and is upper-bounded as $\wb{\kappa} \lesssim \max\Big\{\frac{d^3H^4}{\lambda_+^{2}},\frac{d^2H^4}{\Delta_{\min}^{2}\lambda_+^{3}} \Big\}$, where $\lambda_+ := \min_h \{\lambda_h^+\} > 0$.
\end{theorem}

In both cases, $\wb{\kappa}$ is polynomial in all the problem-dependent terms and independent of the number of episodes $K$ (see Lem.~\ref{lem:kappabound_ele} and~\ref{lem:kappabound_lsvi}). As a result, ELEANOR and LSVI-UCB achieves a constant regret that only depends on ``static'' MDP and representation characteristics, thus indicating that after a finite time the agent only executes the optimal policy. Notice also that the bounds should be read as minimum between the constant regret and the minimax regret $O(\sqrt{K})$, which may be tighter for small $K$.
The main difference between the two previous bounds is that for ELEANOR we build on the anytime minimax regret bound, while for LSVI-UCB, we derive a more refined constant-regret guarantee by building on its problem-dependent bound of~\citet{He2020loglin}. Unfortunately, limiting factor for applying the analysis in~\citep{He2020loglin} seems to be the fact that ELEANOR is not optimistic at each stage $h$ but rather only at the first stage. As such, whether ELEANOR can achieve a problem-dependent logarithmic regret based on local gaps that can be leverage to improve our analysis is an open question in the literature.

\paragraph{Combined proof sketch of Thm.~\ref{thm:const.ele} and Thm.~\ref{thm:const.lsvi}.}
We provide a general proof sketch that can be instantiated to both ELEANOR and LSVI-UCB. The purpose is to illustrate what properties an algorithm must have to exploit good representations, and how this leads to constant regret. Consider a learnable feature map $\{\phi_h\}_{h\in[H]}$ and an algorithm with the following properties:
\begin{enumerate}[label=(\alph*),noitemsep,topsep=0pt,parsep=0pt,partopsep=0pt,leftmargin=0.3in]
    \item \label{p:greedy}
    Greedy w.r.t.\ a Q-function estimate: $\pi^k_h(s) = \arg\max_{a\in\Aspace}\{\overline{Q}_h^k(s,a)\}$.
    \item \label{p:global.opot}
    Global optimism: $\overline{V}_1^k(s) \ge V_1^\star(s)$ where,  for all $h \geq 1$, we set $\overline{V}_h^k(s)=\max_{a\in\Aspace}\{\overline{Q}_h^k(s,a)\}$.
    \item \label{p:local.opt} 
    Almost local optimism: $\forall h>1, \exists C_h\ge0$ s.t.\ 
    $\overline{Q}^k_h(s,a)+C_h\beta_k\norm{\phi_h(s,a)}_{(\Lambda_h^k)^{-1}} \ge Q^\star_h(s,a)$. 
    \item \label{p:confidence}
    Confidence set: let $\Lambda_h^{k}=\sum_{i=1}^{k-1}\phi_h(s_h^i,a_h^i)\phi_h(s_h^i,a_h^i)^\transp + \lambda I$ and $\beta_k \in \mathbb{R}_+$ be logarithmic in $k$, then
    $\overline{V}_h^k(s_h^k) - V_h^{\pi_k}(s_h^k) \le 2\beta_k\norm{\phi_h(s_h^k,a_h^k)}_{(\Lambda_h^k)^{-1}} + \EV_{s'\sim p_h(s_h^k,a_h^k)}\left[\overline{V}_{h+1}^k(s')-V^{\pi_k}_{h+1}(s')\right]$.
\end{enumerate}
These properties are verified by ELEANOR~\citep[][App.\ C]{Zanette2020low} and LSVI-UCB~\citep[][Lem. B.4, B.5]{Jin2020linear}. Note that for LSVI-UCB condition~\ref{p:local.opt}  is trivially verified since the algorithm is optimistic at each stage ($C_h=0$). On the other hand, ELEANOR is only guaranteed to be optimistic at the first stage, and~\ref{p:local.opt} is thus important ($C_h=2$).
First, we use existing techniques to establish an \emph{any-time} regret bound, either worst-case or problem-dependent. We call this $g(k)$ and prove that $R(k)\le g(k) \le \wt{O}(\sqrt{k})$ for any $k$ with probability $1-2\delta$. 

Next, we show that, under Asm.~\ref{asm:feature.structure}, the eigenvalues of the design matrix grow almost linearly, making the confidence intervals decrease at a $1/\sqrt{k}$ rate. From some algebra and a martingale argument,
\begin{align}\label{eq:design}
    \Lambda_h^{k+1} \succeq k\Lambda_h^\star + \lambda I - \Delta_{\min}^{-1}g(k) I - \wt{O}(\sqrt{k})I,
\end{align}
where $\Lambda_h^\star=\EV_{s\sim\rho_h^\star}[\phi_h^\star(s)\phi_h^\star(s)^\transp]$.
The \hls property ensures that the linear term is nonzero in relevant directions, while the regret bound of the algorithm makes the penalty term sublinear. Then, we show that, for any \emph{reachable} $(s,a)$,
\begin{equation}\label{eq:shrink}
    \beta_k\norm{\phi_h(s,a)}_{(\Lambda_h^k)^{-1}} \le \beta_k \frac{k-\wt{O}(\sqrt{k})}{(k\lambda_{h}^+ -\wt{O}(\sqrt{k}))^{3/2}} = \wt{O}(k^{-1/2}),
\end{equation}
where $\lambda_{h}^+$ is the minimum \emph{nonzero} eigenvalue of $\Lambda_h^\star$. From~\eqref{eq:shrink}, we can see that $\lambda_{h}^+$ plays a fundamental role in the rate of decrease.
Finally, we show that, under the gap assumption, these uniformly-decreasing confidence intervals allow learning the optimal policy in a finite time. From the Bellman equations, we have that
\begin{equation}\label{eq:sum.of.gaps}
    V_1^\star(s_1^k) - V_1^{\pi^k}(s_1^k) = \EV_{\pi^k}\left[\sum_{h=1}^H\Delta_h(s_h,a_h)|s_1=s_1^k\right],
\end{equation}
while from~\ref{p:greedy}-\ref{p:confidence}, for any reachable state,
\begin{equation*}
\begin{aligned}
    \Delta_h(s,\pi^k_h(s)) \le 2\EV_{\pi^k}\left[\sum_{i=h}^H \beta_k\norm{\phi_i(s_i,a_i)}_{(\Lambda_i^k)^{-1}}|s_h=s\right] 
    + \mathds{1}_{h>1}
    C_h\beta_k\norm{\phi_h^\star(s)}_{(\Lambda_h^k)^{-1}}.
\end{aligned}
\end{equation*}
The second term (with $\mathds{1}_{h>1}$) accounts for the almost-optimism of ELEANOR, while it is zero in LSVI-UCB due to the stage-wise optimism.
Then, for every $h\in[H]$, we can use~\eqref{eq:shrink} to control the feature norms. 
Thus, there exists an episode $\kappa_h$ independent of $K$ satisfying
\begin{align}\label{eq:tau}
    \Delta_h(s,\pi^k_h(s)) &\leq \beta_{\kappa_h}\sum_{i=h}^H(2+\mathds{1}_{i=h>1}C_h)\frac{\kappa_h -8\sqrt{\kappa_h\log(2d\kappa_hH/\delta)} - g(\kappa_h)}{(\kappa_h\lambda_{i}^{+} - 8\sqrt{\kappa_h\log(2d\kappa_h H/\delta)} - g(\kappa_h))^{3/2}} 
    < \Delta_{\min},
 \end{align}
%
By definition of minimum gap, then $\Delta_h(s,\pi^k_h(s))=0$ for $k>\kappa_h$. 
Then, for $k>\wb{\kappa}=\max_{h}\{\kappa_h\}$, $V_1^\star(s_1^k) - V_1^{\pi^k}(s_1^k)=0$. But this means the algorithm only accumulates regret up to $\wb{\kappa}$, that is, $R(K)=R(\wb{\kappa})\le g(\wb{\kappa})=O(1)$ for all $K>\wb{\kappa}$. This holds with probability $1-3\delta$, also taking into account the martingale argument from~\eqref{eq:design}. Note that $\{\kappa_h\}$ are by definition monotone for LSVI-UCB.
The final bounds are then obtained by instantiating the specific values of $\beta_k$ and $g(k)$ for the two algorithms we analyzed.

\section{Representation Selection in Low-Rank MDPs}\label{sec:rep.selection}
In Sec.~\ref{sec:constant}, we have highlighted the benefits that a \hls{} representation brings to optimistic algorithms in MDPs with Bellman closure and low rank structure. In this section, we take one step further and investigate the \emph{representation selection} problem. 
Since ELEANOR is a computationally intractable algorithm, we build on LSVI-UCB and low-rank MDPs (Asm.~\ref{asm:lowrank}) and we introduce LSVI-LEADER (Alg.~\ref{alg:LSVI-LEADER}), an algorithm that adaptively selects representations in a given set.

\begin{algorithm}[t]
\footnotesize
\KwIn{Representations $\{\Phi_j\}_{j\in [N]}$, confidence values $\{\beta_k\}_{k\in [K]}$}
\For{$k=1, \ldots, K$ }{
Receive the initial state $s_1^k$\\
\For{$h=H, \ldots, 1$}{
 $\Lambda_h^k(j)=\lambda I + \sum_{i=1}^{k-1}\phi_h^{(j)}(s_h^i,a_h^i)\phi_h^{(j)}(s_h^i,a_h^i)^\transp$ $\forall$ $j \in [N]$. \\
 $\boldsymbol{w}_h^k(j) = \Lambda_h^{k}(j)^{-1}\sum_{i=1}^{k-1}\phi_h^{(j)}(s_h^i,a_h^i)\left(r_h(s_h^i,a_h^i)+\max\limits_{a\in\Aspace}\overline{Q}^{k}_{h+1}(s_{h+1}^i,a)\right),~\forall j \in [N]$ \\
 $\overline{Q}_h^k(s,a) = \min\left\{H,\min_{j \in [N]}\left( \phi_h^{(j)}(s,a)^\transp \boldsymbol{w}_h^k(j) + \beta_{k}\norm{\phi_h^{(j)}(s,a)}_{\Lambda_h^k(j)^{-1}} \right)\right\}$
}
\For{$h=1, \ldots, H$}{
Execute action $a_h^k=\pi_h^k(s_h^k) := \argmax_{a\in\mathcal{A}}\overline{Q}_h^k(s_h^k,a)$. 
 }
}

\caption{LSVI-LEADER}
\label{alg:LSVI-LEADER}
\end{algorithm}

Given a set of $N$ representations $\{\Phi_j\}_{j\in[N]}$ satisfying Asm.~\ref{asm:lowrank}, where $\Phi_j = \big\{\phi_h^{(j)}\big\}_{h\in[H]}$, 
at each stage $h\in[H]$ of episode $k\in[K]$, LSVI-LEADER solves $N$ different regression problems to compute an optimistic value function for each representation. Then, the final estimate $\overline{Q}_h^k(s,a)$ is taken as the \emph{minimum} across these different optimistic value functions. Notably, this implies that LSVI-LEADER implicitly \emph{combines} representations, in the sense that the selected representations (i.e., those with tightest optimism) might vary for different 
stages. This is exploited in the following result, which shows that constant regret is achievable even if none of the given representations is globally \hls{}.
\begin{theorem}\label{thm:mix.stage}
    Given an MDP $M$ and a set of representations $\{\Phi_j\}_{j\in[N]}$ satisfying the low-rank assumption (Asm.~\ref{asm:lowrank}), let $\mathcal{Z}$ be the set of $H^N$ representations obtained by combining those in $\{\Phi_j\}_{j\in[N]}$ across different stages.\footnote{Note that any combination of features in $\Phi_j$ is learnable, since each representation is learnable in the low-rank MDP sense.} Then, with probability at least $1-2\delta$, LSVI-LEADER suffers at most a regret
    \[
    R(K) \leq \min_{z\in\mathcal{Z}}\wt{R}(K, z, \{\beta_k\}),
    \]
    where $\wt{R}(K, z, \beta_k)$ is either the worst-case regret bound of LSVI-UCB \citep{Jin2020linear} or the problem-dependent one \citep{He2020loglin} when the algorithm is executed with representation $z$ and confidence values $\beta_k\propto dH\sqrt{N\log(2dNHk/\delta)}$. Moreover, if $\mathcal Z$ contains a \hls{} representation $z^\star$, then LSVI-LEADER achieves constant regret with problem-dependent values of $z^\star$ (see Thm.~\ref{thm:const.lsvi}).
\end{theorem}
This result shows that LSVI-LEADER adapts to the \emph{best} representation automatically, i.e., without any prior knowledge about the properties of the representations. In particular, it shows a problem-dependent (or worst-case) bound when there is no \hls{} representation, while it attains constant regret when a representation, potentially mixed through stages, is \hls{}.
This is similar to what was obtained by~\citet{papini2021leveraging} for linear contextual bandits. Indeed, LSVI-LEADER reduces to their algorithm in the case $H=1$. While the cost of representation selection is only logarithmic in linear bandits, the cost becomes polynomial (i.e., $\sqrt{N}$ in the worst-case bound and $N$ in the problem-dependent one) in RL. This is due to the structure induced by the Bellman equation, which requires a cover argument over $H^N$ functions (more details in the proof sketch). Note that for $H=1$, the analysis can be refined to obtain a $\log(N)$ dependence, due to the lack of propagation through stages, and recover the result in~\citep{papini2021leveraging}. We refer the read to App.~\ref{app:examples.experiments} for a numerical validation.

\paragraph{Proof sketch of Thm.~\ref{thm:mix.stage}.} The proof relies on the following important result, which extends Lem.~B.4 of \cite{Jin2020linear} and shows that the deviation between the optimistic value function computed by LSVI-LEADER and the true one scales with the \emph{minimum} confidence interval across the different representations. Formally, with probability $1-2\delta$, for any $\pi\in\Pi,s\in\mathcal{S},a\in\mathcal{A},h\in[H],k\in[K]$,
\begin{align*}
    \overline{Q}_h^k(s,a) - Q_h^{\pi}(s,a) \le 2\beta_k\min_{j\in[N]} \norm{\phi_h^{(j)}(s,a)}_{\Lambda_h^k(j)^{-1}} + \EV_{s'\sim p_h(s,a)}\left[\overline{V}_{h+1}^k(s')-V^{\pi}_{h+1}(s')\right].
\end{align*}
As in \citep{Jin2020linear}, the derivation of this result combines the well-known self-normalized martingale bound in \citep{abbasi2011improved} with a covering argument over the space of possible optimistic value functions. In our setting, the structure of such function space requires us to build $N$ different covers, one for each different representation. This, in turn, requires the confidence values $\beta_k$ to be inflated by an extra factor $\sqrt{N}$ w.r.t.\ learning with a single representation. 

The generality of this result allows us to easily derive, for any fixed representation $z\in\mathcal{Z}$, both the worst-case regret bound of \cite{Jin2020linear} and the problem-dependent one of \cite{He2020loglin}. To see this, note that the regret decompositions in both of these two papers rely on an upper bound to $\overline{V}_h^k(s_h^k) - V_h^{\pi_k}(s_h^k)$ as a function of the \emph{fixed} representation used by LSVI-UCB (see the proof of Theorem 3.1 of \cite{Jin2020linear} and Lemma 6.2 of \cite{He2020loglin}). Then, fix any $z\in\mathcal{Z}$ and call $z_h$ its features at stage $h$. Note that $z_h\in\{\phi_h^{(j)}\}_{j\in[M]}$. Moreover, by definition of low-rank structure, since each $\Phi_j$ induces a low-rank MDP, their combination does too. Thus, $z$ is learnable. Then, instantiating the concentration bound stated above for policy $\pi^k$, state $s_h^k$, action $a_h^k$, stage $h$, and by upper bounding the minimum with the representation selected in $z_h$, we get
\begin{align*}
    \overline{V}_h^k(s_h^k) - V_h^{\pi_k}(s_h^k) \le 2\beta_k \norm{z_h(s_h^k,a_h^k)}_{\Lambda_h^k(j)^{-1}} + \EV_{s'\sim p_h(s_h^k,a_h^k)}\left[\overline{V}_{h+1}^k(s')-V^{\pi_k}_{h+1}(s')\right].
\end{align*}
From here, one can carry out exactly the same proofs of \cite{Jin2020linear} and \cite{He2020loglin}, thus obtaining the same regret bound that LSVI-UCB enjoys when executed with the fixed representation $z\in\mathcal{Z}$ and confidence values $\{\beta_k\}_{k\in[K]}$. Hence, we conclude that the regret of LSVI-LEADER is upper bounded by the minimum of these regret bounds for all representations $z\in\mathcal{Z}$, thus proving the first result. To obtain the second result, simply notice that, if $z^\star \in \mathcal{Z}$ is \hls{}, then we can use the refined analysis for LSVI-UCB of Thm.~\ref{thm:const.lsvi} to show that $\wt{R}(K,z^\star,\{\beta_k\})$ is upper bounded by a constant independent of $K$, hence proving constant regret for LSVI-LEADER.

\subsection{Representation Selection under a Mixing Condition}

We show that the LSVI-LEADER algorithm not only is able to select the best representation among a set of viable representations, and to combine representations for the different stages, but also to stitch representations together \emph{across states and actions}.
With this in mind we introduce the notion of a mixed ensemble of representations. 

\begin{definition}\label{def:mixing-hls}
 Consider an MDP $M$ and a set of representations $\{\Phi_j\}_{j\in[N]}$ satisfying the low-rank assumption (Asm.~\ref{asm:lowrank}). The collection of feature maps $\{ \Phi_j\}_{j\in[M]}$ is \hls{}-mixing if  for all $s,a \in \mathcal{S}\times \mathcal{A}$ and $h \in [H]$, there exists $j$ such that $\phi^{(j)}_h(s,a) \in \mathrm{span}\left\{\phi_h^{(j)}(s, \pi^\star_h(s)) | \rho^{\star}_h(s) > 0\right\}$.
\end{definition}

We show that when presented with a \hls{}-mixing family of representations, LSVI-LEADER is able to successfully combine these and obtain a regret guarantee that may be better than what is achievable by running LSVI-UCB using any of these representations in isolation.

\begin{theorem}\label{thm:mix.state}
Consider an MDP $M$ and a set of representations $\{\Phi_j\}_{j\in[N]}$ satisfying the low-rank  (Asm.~\ref{asm:lowrank}) and \hls{}-mixing assumptions. If $\Delta_{\min} >0$ (Asm.~\ref{asm:gap}), then with probability at least $1-3\delta$, there exist a constant $\wt{\kappa} = \max_{h} \{\kappa_h\}$ independent from $K$ such that the regret of LSVI-LEADER after $K$ episodes is at most:
    \begin{equation*}
        R(K) \leq \min_{z\in\mathcal{Z}}\wt{R}\big( \wt{\kappa}, z, \{\beta_k\} \big),
    \end{equation*}
    where $\mathcal{Z}$, $\wt{R}$ and $\beta_k$ are defined as in Thm.~\ref{thm:mix.stage}. 
\end{theorem}
Under the \hls{}-mixing condition, LSVI-LEADER may not converge to selecting a single representation for each stage $h$ but rather to mixing multiple representations. In fact, it may select a different representation in different regions of the state-action space. 
This is the main difference w.r.t.\ Thm.~\ref{thm:mix.stage}, where constant regret is shown when there exists a representation $z^\star$ that is \hls{}, and the value $\kappa_h$ depends on the minimum positive eigenvalue of $z^\star_h$. In the case of \hls{}-mixing, $\kappa_h$ depends on properties of a combination of representations at stage $h$. We provide a characterization of $\kappa_h$ in the full proof in App.~\ref{app:rep.selection}.

\section{Conclusions}\label{sec:conclusions}
We investigated the properties that make a representation efficient for online learning in MDPs with Bellman closure. We introduced \hls{}, a necessary and sufficient condition to achieve a constant regret bound in this class of MDPs. We demonstrate that existing optimistic algorithms are able to adapt to the structure of the problem and achieve constant regret. Furthermore, we introduce an algorithm able to achieve constant regret by mixing representations across states, actions and stages in the case of low-rank MDPs.
An interesting direction raised by our paper is whether it is possible to leverage the \hls{} structure for probably-efficient representation learning, rather than selection. Another direction can be to leverage these insights to drive the design of auxiliary losses for representation learning, for example in deep RL.

\bibliography{bibliography}
\bibliographystyle{plainnat}

\clearpage
\begin{appendix}
\addcontentsline{toc}{section}{Appendix} 
\part{Appendix} 
\parttoc 

\section{Related Work} \label{app:related.work}
The representation selection problem has been originally studied in the context of tabular MDPs. Given a set of representation mapping histories to (sequences of actions, observations, and rewards) to a finite set of states, the goal of the learning agent is to solve the MDP under an appropriate representation. The standard assumption is that at least one representation induces an MDP. Several papers have investigated this online learning problem and provided algorithms based on the optimism principle~\citep[e.g.,][]{OrtnerMR14,OrtnerPLFM19}. The settings and the representation learning objective are different from ours. In particular, this line of research aims at finding any representation that is good for learning but the methods are not guaranteed to find the most efficient. 

Recently, a few papers have focused on representation learning with theoretical guarantees. \citet{DuKJAD019} considered the representation learning problem in block MDPs with rich observations, where the objective is to learn the compact latent representation. Representation learning in low-rank MDPs was recently studied in~\citep{AgarwalKKS20,Modi2021modelfree,lu2021power}. We believe that these papers are orthogonal to our work for several reasons. 
We start considering the setting in~\citep{AgarwalKKS20,Modi2021modelfree}. First, they operate in the reward-free setting where the objective is to learn a representation of the low-rank MDP that can be used to efficiently learn an optimal policy once a reward is given. For us, a reward is given from the start and learning/selecting a good representation in the meantime is just a way to suffer less regret.
Second, our representation selection objective is different (and arguably more challenging) than the one considered by~\citet{AgarwalKKS20,Modi2021modelfree}.
They aim at finding a representation with low mean square error, i.e. any realizable representation of the low-rank MDP. On the other hand, we wish to find a \hls{} representation among a set of realizable representations, which makes the representation learning problem harder. In App. F, \citet{papini2021leveraging} showed that reducing the MSE is not enough for this purpose. It is shown that this only allows the algorithm to end up with a set of realizable representations, but after that, a different algorithmic scheme, whose primary objective is reducing regret (like LSVI-LEADER), is needed to find the \hls{} one. Therefore, even if the approach in these papers could be extended to the regret minimization setting, there would be no guarantee that running that algorithm would recover a \hls{} representation as in our case. Finally, it is unclear how to transform their sample complexity into a regret bound. In particular, it is not just a matter of translating a sample complexity bound into a regret bound: both the interaction protocol and the algorithmic schemes are different w.r.t.\ our work. Even if we directly translated the sample complexity bounds of these papers into regret bounds, we note that, while it is true that they could scale as $\log(|\Phi|)$, they also contain several dependencies which are orders of magnitude worse than in our work. For instance, the sample complexity provided in~\citep[][Thm. 2]{Modi2021modelfree} scales as $A^{13}$ ($A$ is the number of actions). This is also an unreasonable dependence in any case of practical interest we can think of.
Finally, \citet{lu2021power} studied the effect of representation learning on the sample complexity in multi-task settings, which is quite different from the single-task regret minimization problem considered in this paper. 

\section{Notation}\label{app:notation}

\begin{table}[h]
    \centering
    \caption{Notation.}
    \begin{tabular}{lll}
    \hline
       $\Sspace$ && state space \\
       $\Aspace$ && action space \\
       $H$ && episode length \\
       $r_h$ && reward function at stage $h$ \\
       $p_h$ && transition function at stage $h$\\
       $\mu$ && initial-state distribution\\
       $K$ && number of episodes \\
       $T$ &$=$& $HK$, total number of interactions \\
       $\pi_h$ && policy for stage $h$\\
       $\Pi$ && policy space\\
       $Q_h^\pi$ && state-action value function of policy $\pi$ at stage $h$\\
       $V_h^\pi$ &$=$& $Q_h^\pi(s,\pi_h(s))$ \\
       $\pi^\star_h$ && optimal policy for stage $h$ \\
       $Q^\star_h$ &$=$& $Q^{\pi^\star}_h$, optimal value function at stage $h$ \\
       $V^\star_h(s)$ &$=$& $\max_{a\in\Aspace}Q^\star_h(s,a)$ \\
       $L_h$ && Bellman's optimality operator for stage $h$ \\
       $\pi_h^k$ && policy played by the algorithm at stage $h$ of episode $k$ \\
       $\phi_h$ && feature map for stage $h$ \\
       $R(K)$ && regret suffered in the first $K$ episodes \\
       $d$ && feature dimension \\
       $D$ &$=$& $H$, value function upper bound \\
       $\mathcal{Q}_h$ && set of linear bounded value functions for stage $h$ \\
       $\Theta_h$ && set of parameters of linear bounded value functions for stage $h$ \\
       $\Delta_h(s,a)$ &$=$& $V_h^\star(s) - Q_h^\star(s,a)$, suboptimality gap  \\
       $\Delta_{\min}$ && minimum positive gap (see Asm.~\ref{asm:gap}) \\
       $\phi^\star_h(s)$ &$=$& $\phi_h(s,\pi_h^\star(s))$, optimal features for state $s$ at stage $h$\\
       $\rho_h^{\pi}$ && occupancy measure of policy $\pi$ at stage $h$ (see Asm.~\ref{asm:feature.structure}) \\
       $\Lambda_h^\star$ &$=$& $\EV_{s\sim\rho_h^\star}[\phi_h^\star(s)\phi_h^\star(s)]$, optimal covariance matrix \\
       $\lambda_h^+$ && minimum nonzero eigenvalue of $\Lambda_h^\star$ \\
       $\delta$ && failure probability \\
       $\overline{\kappa}$ && last episode at which nonzero regret is paid (see proof of Thm.~\ref{thm:general.const}) \\
       $\overline{\tau}$ &$=$& $H\overline{\kappa}$ \\
       $\beta_k$ && confidence radius, see~\eqref{eq:beta.eleanor} for ELEANOR and~\eqref{eq:beta.lsvi} for LSVI-UCB \\
       $\lambda$ &$=$& $1$, regularization parameter \\
       $\Lambda_h^k$ &$=$& $\lambda I + \sum_{i=1}^{k-1}\phi_h(s_h^i,a_h^i)\phi_h(s_h^i,a_h^i)^\transp$, design matrix \\
       $\overline{Q}^k_h$ && optimistic value function for stage $h$ at episode $k$ \\
       $\overline{V}_h^k(s)$ &$=$& $\max_{a\in\Aspace}\overline{Q}_h^k(s,a)$\\
       \hline
    \end{tabular}
    \label{tab:notation}
\end{table}
\clearpage
\section{\hls{} is Necessary: Proofs of Section \ref{sec:necessary}}\label{app:necessary}

We illustrate all the detailed proofs needed for showing that the \hls{} condition is necessary to achieve constant regret (Thm.~\ref{thm:necessity}). For the sake of completeness, we restate here all the assumptions on the MDP $M$ under consideration.

\paragraph{Assumptions on MDP $M$.}

\begin{itemize}
\item $\mathcal{S}$ and $\mathcal{A}$ finite, $H \geq 1$ arbitrary;
\item Linear rewards: $r_h(s,a) = \langle \theta_h, \phi(s,a)\rangle$ with $\mathcal{N}(0,1)$ noise;
\item Arbitrary transition probabilities $\{p_h\}_{h\in[H]}$ and initial-state distribution $\mu$;
\item Unique optimal policy $\pi^\star$: $|\{a : Q^\star_h(s,a)=V^\star_h(s)\}| = 1$ and $\pi_h^\star(s) = \argmax_{a}Q^\star_h(s,a)$ for all $s,h$;
\item \hls{} condition (Asm.~\ref{asm:feature.structure} does not hold).
\end{itemize}
Moreover, recall that we define $\mathcal{M}$ as any set of MDPs that contains (but it can be larger than) all the MDPs which are equivalent to $M$ in all components except for the reward parameters $\{\theta_h\}_{h\in[H]}$, which can be arbitrary vectors in $\mathbb{R}^d$. Formally,
\begin{align*}
    \mathcal{M} \supseteq \left\{ \wt{M} = \left(\Sspace, \Aspace, H, \{\wt{r}_h\}_{h=1}^H, \{p_h\}_{h=1}^H, \mu\right) \mid \forall h\in[H],\exists \wt{\theta}_h \in \mathbb{R}^d : \wt{r}_h(s,a) = \langle \wt{\theta}_h, \phi(s,a)\rangle\right\}.
\end{align*}
Intuitively, $\mathcal{M}$ contains at least all the MDPs that could be faced by an agent that knows the linear-reward structure of the problem but that does not know the true parameters $\{\theta_h\}_{h\in[H]}$. Obviously, if the agent knows all the components of $M$ except for the reward parameters, the set $\mathcal{M}$ can be taken exactly as the set on the righthand side above (which would contain all and only the realizable MDPs). On the other hand, in the more general case where the agent does not know the dynamics as well, set $\mathcal{M}$ can be enlarged by including all the realizable MDPs with different transition probabilities (e.g., those with low-rank or low-IBE structure, or even the whole set of unstructured dynamics). Our proof that \hls{} is necessary for constant regret holds for an agent that only knows that the true MDP $M$ belongs to this general set $\mathcal{M}$ and thus encompasses all the relevant settings mentioned in Sec.~\ref{sec:necessary}.

In the following proofs we shall write $\mathbb{P}_{M}^{\mathsf{A}}$ ($\mathbb{E}_{M}^{\mathsf{A}}$) to denote the probability (expectation) operator under MDP $M$ and the chosen algorithm ${\mathsf{A}}$.

\subsection{Proof of Lemma~\ref{lemma:bound-psi}}

Let ${M}$ be our true MDP and $\wt{{M}} \in \mathcal{M}$ be any other MDP which is equivalent to $M$ in all components except for the reward parameters, which are given by $\{\wt{\theta}_h\}_{h\in[H]}$. We start by a standard decomposition of the expected log-likelihood ratio between the observations generated in the two MDPs. Fix $K \geq 1$ and let $\mathrm{KL}(\mathbb{P}_{{M}}, \mathbb{P}_{\wt{{M}}})$ denote the KL-divergence between the distributions of the observations collected by algorithm ${\mathsf{A}}$ over $K$ episodes. Using, e.g., Lemma 5 of \cite{domingues2021episodic} together with the closed-form of the KL divergence between Gaussian distributions,
\begin{align*}
\mathrm{KL}(\mathbb{P}_{{M}}, \mathbb{P}_{\wt{{M}}}) = \sum_{s,a}\sum_{h\in[H]}\mathbb{E}_{{M}}^{\mathsf{A}}[N_h^K(s,a)] \frac{(\langle \phi(s,a), \theta_h-\wt{\theta}_h \rangle)^2}{2} = \frac{1}{2}\sum_{h\in[H]} \|\theta_h - \wt{\theta}_h\|_{\mathbb{E}_{{M}}^{\mathsf{A}}[\Lambda_h^K]}^2,
\end{align*}
where $\Lambda_h^K := \sum_{s,a} N_h^K(s,a)\phi(s,a)\phi(s,a)^T$ and $N_h^K(s,a) := \sum_{k=1}^K\indi{s_h^k=s,a_h^k=a}$. 

Suppose that, for sufficiently large $K$, the matrix $\mathbb{E}_{{M}}^{\mathsf{A}}[\Lambda_h^K]$ is invertible.\footnote{\cite{lattimore2017end} proved that this is indeed true for consistent algorithms. Otherwise, one could simply make the matrix positive-definite by adding $\lambda I$ for some arbitrary $\lambda > 0$ and the derivation still holds.} 
We now proceed as follows. For a fixed $h\in[H]$ and sub-optimal policy $\pi\in\Pi$ (i.e., with $\Delta(\pi) > 0$), we seek the hardest MDP $\wt{M}$ to discriminate from ${M}$ (i.e., that minimizes $\mathrm{KL}(\mathbb{P}_{{M}}, \mathbb{P}_{\wt{{M}}})$) where policy $\pi$ is strictly better (in terms of expected return) than $\pi^\star$ and where we change only the parameter $\theta_h$ w.r.t. ${M}$. Formally, we minimize
\begin{align*}
\mathrm{minimize}_{\wt{\theta}_h \in \mathbb{R}^d}\|\theta_h - \wt{\theta}_h\|_{\mathbb{E}_{{M}}^{\mathsf{A}}[\Lambda_h^K]}^2
\end{align*}
subject to the constraint $\wt{V}_1^\pi \geq \wt{V}_1^{\pi^\star} + \epsilon$. First note that the expected return of policy $\pi$ can be equivalently written as
\begin{align*}
V_1^\pi = \sum_{s,a}\sum_{h\in[H]}\rho_h^\pi(s,a)r_h(s,a) = \sum_{h\in[H]}\langle \theta_h, \sum_{s,a}\rho_h^\pi(s,a)\phi(s,a)\rangle = \sum_{h\in[H]}\langle \theta_h, \Psi_h^\pi\rangle.
\end{align*}
Moreover, since ${M}$ and $\wt{{M}}$ have same transition probabilities, $\Psi_h^\pi = \wt{\Psi}_h^\pi$ for each $\pi,h$. Thus, $\wt{V}_1^\pi = \sum_{h\in[H]}\langle \wt{\theta}_h, \Psi_h^\pi\rangle$ and the constraint can be rewritten in the more convenient form
\begin{align*}
\sum_{h\in[H]}\langle \wt{\theta}_h, \Psi_h^\pi\rangle \geq \sum_{h\in[H]}\langle \wt{\theta}_h, \Psi_h^{\star}\rangle + \epsilon.
\end{align*}
Using Lemma \ref{lemma:closed-form-alternative}, the optimization problem has a closed-form expression. Therefore, let $\Gamma_h^\epsilon(\pi) \subseteq \mathcal{M}$ be the set of MDPs over which we are optimizing, that is, with (1) same transition probabilities as $\mathcal{M}$, (2) same reward parameters as $\mathcal{M}$ at all stages except $h$, and (3) $\wt{V}_1^\pi \geq \wt{V}_1^{\pi^\star} + \epsilon$. Using Lemma \ref{lemma:closed-form-alternative} together with the rewritings above, for any $\pi\in\Pi, h\in[H]$ and $\epsilon \geq 0$,
\begin{align}\label{eq:aaaaa}
\min_{\wt{{M}} \in \Gamma_h^\epsilon(\pi)} \mathrm{KL}(\mathbb{P}_{{M}}, \mathbb{P}_{\wt{{M}}}) = \frac{ \left(\Delta(\pi)+\epsilon\right)^2}{2\|\Psi_h^\pi - \Psi_h^{\star}\|_{\mathbb{E}_{{M}}^{\mathsf{A}}[\Lambda_h^K]^{-1}}^2}.
\end{align}

We now show that $\mathrm{KL}(\mathbb{P}_{{M}}, \mathbb{P}_{\wt{{M}}})$ is lower bounded by a quantity that increases logarithmically in $K$ for any $\wt{{M}} \in \Gamma_h^\epsilon(\pi)$ with $\epsilon > 0$.
Let $E_K := \{\sum_{\pi\in\Pi^\star} N_K(\pi) < f(K)\}$, where $N_K(\pi) := \sum_{k=1}^K \indi{\pi^k = \pi}$, $\Pi^\star$ is the set of all deterministic policies with maximal expected return in $M$, and $f(K)$ will be specified later. Using Lemma \ref{lemma:bh-ineq},
\begin{align}\label{eq:che.cazzo.ne.so}
\mathrm{KL}(\mathbb{P}_{M}, \mathbb{P}_{\wt{M}}) \geq \log\frac{1}{\mathbb{P}_{M}(E_K) + \mathbb{P}_{\wt{M}}(E^c)} -\log 2.
\end{align}
Now note that, under the assumption that $\mathsf{A}$ is $\alpha$-consistent, 
\begin{align*}
c_M K^\alpha \geq \mathbb{E}_{{M}}^{\mathsf{A}}\left[{R}(K) \right] = \sum_{\pi\in\Pi} \mathbb{E}_{{M}}^{\mathsf{A}}\left[N_K(\pi)\right]\Delta(\pi)
\geq \Delta \sum_{\pi\notin\Pi^\star}\mathbb{E}_{{M}}^{\mathsf{A}}\left[N_K(\pi)\right] .
\end{align*}
Here, with some abuse of notation, $\Delta$ is the minimum policy gap. Therefore,
\begin{align*}
\mathbb{P}_{{M}}(E_K) = \mathbb{P}_{{M}}\left(K - \sum_{\pi\notin\Pi^\star}N_K(\pi) < f(K)\right) \leq \frac{\sum_{\pi\notin\Pi^\star}\mathbb{E}_{{M}}^{\mathsf{A}}\left[N_K(\pi)\right]}{K-f(K)} \leq \frac{K^\alpha c_M/ \Delta }{K-f(K)},
\end{align*}
where the first inequality is Markov's inequality. Note that, since $\Psi_h^\pi = \Psi_h^\star$ for all optimal policies $\pi\in\Pi^\star$ and since the transition probablities of ${M}$ and $\wt{{M}}$ are the same, $\wt{V}_1^\pi = \wt{V}_1^{\pi^\star}$ for all $\pi\in\Pi^\star$. Hence, all optimal policies for ${M}$ have a gap of at least $\epsilon$ in $\wt{{M}}$. This implies that
\begin{align*}
c_{\wt{M}}K^\alpha \geq \mathbb{E}_{\wt{{M}}}^{\mathsf{A}}\left[R(K) \right]
\geq \epsilon \mathbb{E}_{\wt{{M}}}^{\mathsf{A}}\left[\sum_{\pi\in\Pi^\star}N_K(\pi)\right] .
\end{align*}
Therefore,
\begin{align*}
\mathbb{P}_{\wt{{M}}}(E_K^c) = \mathbb{P}_{\wt{\mathcal{M}}}\left(\sum_{\pi\in\Pi^\star}N_K(\pi) \geq f(K)\right) \leq \frac{\mathbb{E}_{\wt{M}}^{\mathsf{A}}\left[\sum_{\pi\in\Pi^\star}N_K(\pi)\right]}{f(K)} \leq \frac{K^\alpha c_{\wt{{M}}}/ \epsilon }{f(K)}.
\end{align*}
If we set $f(K) = K/2$ and plug the two bounds above into \eqref{eq:che.cazzo.ne.so}, we obtain
\begin{align*}
\mathrm{KL}(\mathbb{P}_{{M}}, \mathbb{P}_{\wt{{M}}}) \geq \log\frac{K^{1-\alpha}}{2c_M/ \Delta_{} + 2c_{\wt{M}}/ \epsilon} -\log 2.
\end{align*}
Finally, for any $\wt{{M}} \in \Gamma_h^\epsilon(\pi)$ with $\epsilon > 0$,
\begin{align*}
\liminf_{K\rightarrow\infty}\frac{\mathrm{KL}(\mathbb{P}_{{M}}, \mathbb{P}_{\wt{{M}}})}{\log(K)} \geq 1-\alpha.
\end{align*}
This holds for any $\epsilon > 0$. Hence, in combination with \eqref{eq:aaaaa}, we proved that, for any sub-optimal policy $\pi$ and stage $h$,
\begin{align*}
    \liminf_{K\rightarrow\infty}\frac{1}{\log(K)} \frac{ \Delta(\pi)^2}{2\|\Psi_h^\pi - \Psi_h^{\star}\|_{\mathbb{E}_{{M}}^{\mathsf{A}}[\Lambda_h^K]^{-1}}^2} \geq 1-\alpha.
\end{align*}
Rearranging concludes the proof.

\subsection{Proof of Theorem~\ref{thm:necessity}}

We now use Lemma \ref{lemma:bound-psi} to prove that the \hls{} condition is necessary for constant regret. We proceed in different steps.

\paragraph{Step 1. Controlling the design matrix.} 

Suppose that the algorithm suffers constant regret on instance ${M}$. This means that, for some constant $C_M$ (different from the $c_M$ used in the definition of $\alpha$-consistence),
\begin{align}
    \mathbb{E}_{{M}}^{\mathsf{A}}\left[\mathrm{R}(K) \right] \leq C_M.
\end{align}
Since $\mathbb{E}_{{M}}^{\mathsf{A}}\left[\mathrm{R}(K) \right] = \sum_{h}\sum_{s,a}\mathbb{E}_{{M}}^{\mathsf{A}}\left[N_h^K(s,a) \right] \Delta_h(s,a)$, we have that $\sum_{h}\sum_{s,a\neq\pi_h^\star(s)}\mathbb{E}_{{M}}^{\mathsf{A}}\left[N_h^K(s,a) \right] \leq C_{M}/\Delta_{\min}$, where $\Delta_{\min}$ is the minimum value-function gap. Therefore, the expected design matrix at each $h\in[H]$ satifies
\begin{align*}
\mathbb{E}_{{M}}^{\mathsf{A}}[\Lambda_h^K] 
&= \sum_{s,a} \mathbb{E}_{{M}}^{\mathsf{A}}[N_h^K(s,a)]\phi(s,a)\phi(s,a)^T  
\\ &= \sum_{s} \mathbb{E}_{{M}}^{\mathsf{A}}[N_h^K(s,\phi_h^\star(s))]\phi^\star_h(s)\phi_h^\star(s)^T + \sum_{s,a\neq\pi^\star_h(s)} \mathbb{E}_{{M}}^{\mathsf{A}}[N_h^K(s,a)]\phi(s,a)\phi(s,a)^T  
\\ &\preceq \sum_{s} \mathbb{E}_{{M}}^{\mathsf{A}}[N_h^K(s)]\phi^\star_h(s)\phi_h^\star(s)^T + L^2\frac{C_M}{\Delta_{\min}}I
\\ &\preceq K\sum_{s : \rho_h^\star(s) > 0} \phi^\star_h(s)\phi_h^\star(s)^T + \sum_{s : \rho_h^\star(s) = 0} \mathbb{E}_{{M}}^{\mathsf{A}}[N_h^K(s)]\phi^\star_h(s)\phi_h^\star(s)^T + L^2\frac{C_M}{\Delta_{\min}}I.
\end{align*}
We now bound the expected number of times the algorithm visit states which are not visited by an optimal policy. Take any $s$ such that $\rho_h^\star(s) = 0$. Since any optimal policy has the same state distribution $\rho_h^\star$, the event $s_h^k = s$ implies that $\pi^k \notin \Pi^\star$. Therefore,
\begin{align*}
\mathbb{E}_{{M}}^{\mathsf{A}}[N_h^K(s)] = \mathbb{E}_{{M}}^{\mathsf{A}}[\sum_{k=1}^K \indi{s_h^k = s}] \leq \mathbb{E}_{{M}}^{\mathsf{A}}[\sum_{k=1}^K \indi{\pi^k \notin \Pi^\star}] = \mathbb{E}_{{M}}^{\mathsf{A}}[\sum_{\pi\notin\Pi^\star}N_K(\pi)].
\end{align*}
Moreover, since the algorithm suffers constant regret,
\begin{align*}
    \Delta\mathbb{E}_{{M}}^{\mathsf{A}}[\sum_{\pi\notin\Pi^\star}N_K(\pi)] \leq \mathbb{E}_{{M}}^{\mathsf{A}}\left[\mathrm{R}(K) \right] \leq C_M.
\end{align*}
Therefore, we conclude that
\begin{align*}
\mathbb{E}_{{M}}^{\mathsf{A}}[\Lambda_h^K] \preceq
K\sum_{s : \rho_h^\star(s) > 0} \phi^\star_h(s)\phi_h^\star(s)^T + L^2\left(\frac{C_M}{\Delta_{\min}} + S_h\frac{C_M}{\Delta} \right)I,
\end{align*}
where $S_h := S-|\mathrm{supp}(\rho_h^\star))|$.

\paragraph{Step 2. Controlling the feature expectations.}

We now show that, since \hls{} does not hold, there exists a sub-optimal policy $\pi$ such that $\Psi_h^\pi$ is not in the span of the optimal features. By directly using the definition of \hls{} (Asm.~\ref{asm:feature.structure}), we have that there must exist a state-action pair $s,a$ which is reachable at time $h$ (i.e., $\exists \pi\in\Pi : \rho^\pi_h(s,a) > 0$) such that $\phi(s,a) \notin \mathrm{span}\left\{\phi_h^\star(s) | \rho^\star_h(s) > 0\right\}$. Clearly, we have only two cases:
\begin{enumerate}
\item $\rho_h^\star(s) > 0$ and $a \neq \pi_h^\star(s)$;
\item $\rho_h^\star(s) = 0$ and $a$ is arbitrary (even an optimal action).
\end{enumerate}
For Case 1, simply take a policy $\pi$ that is equivalent to $\pi^\star$ everywhere except that $\pi_h(s) = a$. Clearly, the policy is sub-optimal, in the sense that $\Delta(\pi) = V_1^\star - V_1^\pi > 0$. Moreover, it is easy to check that $\Psi_h^\pi - \Psi_h^\star = \rho_h^\star(s) (\phi(s,a) - \phi_h^\star(s))$. Therefore, $\Psi_h^\pi \notin \mathrm{span}\left\{\phi_h^\star(s) | \rho^\star_h(s) > 0\right\}$.

For Case 2, choose $\pi$ in such a way that $\rho_h^\pi(s) > 0$ (we know that one such policy exists due to the reachability of $s$). This only requires selecting the actions of $\pi$ for all stages $h' < h$. For all stages $h' > h$, set $\pi$ equal to $\pi^\star$ except for $\pi_h(s) = a$. Note that, even if $a$ is optimal at time $h$, $\pi$ is strictly sub-optimal (i.e., $\Delta(\pi) > 0$) since no optimal policy can achieve the condition $\rho_h^\pi(s) > 0$ by the uniqueness of the optimal state distribution. Moreover,
\begin{align*}
\Psi_h^\pi - \Psi_h^\star &= 
\sum_{s',a'}\rho_h^\pi(s',a')\phi(s',a') - \sum_{s'}\rho_h^\star(s')\phi_h^\star(s') 
\\ &= \rho_h^\pi(s)\phi(s,a) - \underbrace{\rho_h^\star(s)}_{=0}\phi_h^\star(s) + \sum_{s'\neq s}(\rho_h^\pi(s',a')-\rho_h^\star(s'))\phi_h^\star(s').
\end{align*}
Thus, we still conclude $\Psi_h^\pi \notin \mathrm{span}\left\{\phi_h^\star(s) | \rho^\star_h(s) > 0\right\}$.

\paragraph{Step 3. Concluding the proof.}

Combining Lemma \ref{lemma:bound-psi} with Step 1 and Step 2, we have that, for some $h\in[H]$ and policy $\pi$ such that $\Delta(\pi) > 0$ and $\Psi_h^\pi \notin \mathrm{span}\left\{\phi_h^\star(s) | \rho^\star_h(s) > 0\right\}$,
\begin{align*}
 \limsup_{K\rightarrow\infty} \log(K)\|\Psi_h^\pi - \Psi_h^{\star}\|_{(\Lambda_h^\star + \eta I)^{-1}}^2 \leq \frac{\Delta(\pi)^2}{2(1-\alpha)},
\end{align*}
where $\Lambda_h^\star := K\sum_{s : \rho_h^\star(s) > 0} \phi^\star_h(s)\phi_h^\star(s)^T$ and $\eta := L^2\left(\frac{C_M}{\Delta_{\min}} + S_h\frac{C_M}{\Delta} \right) > 0$. Using Lemma \ref{lemma:abcd}, we have that there exists an $\epsilon > 0$ (independent of $K$) such that $\|\Psi_h^\pi - \Psi_h^{\star}\|_{(\Lambda_h^\star + \eta I)^{-1}} \geq \frac{\epsilon}{\sqrt{\eta}}$. Therefore, we get that
\begin{align*}
 \limsup_{K\rightarrow\infty} \log(K) \leq \frac{\eta\Delta(\pi)^2}{2\epsilon^2(1-\alpha)},
\end{align*}
which clearly does not hold since the left-hand side grows with $K$ while the right-hand side is constant. Therefore, we have a contradiction, and the algorithm $\mathsf{A}$ cannot achieve constant regret on this non-\hls{} instance while being consistent on all other instances in $\mathcal{M}$. Our claim that \hls{} is necessary follows.

\subsection{Auxiliary Results}

\begin{lemma}\label{lemma:closed-form-alternative}
Let $A \in \mathbb{R}^{d\times d}$ be any positive semi-definite invertible matrix. For $\pi\in\Pi$, $h\in[H]$, and $\epsilon \geq 0$, consider the following optimization problem:
\begin{equation*}
    \begin{aligned}
        &\min_{\theta \in \mathbb{R}^d}&& \| \theta - \theta_h \|_{A}^2\\
        &\mathrm{subject\ to} && \sum_{l\in[H], l\neq h}\langle \theta_{l}, \Psi_l^\pi - \Psi_l^{\star}\rangle + \langle \theta, \Psi_h^\pi - \Psi_h^{\star}\rangle \geq \epsilon
    \end{aligned}
\end{equation*}
Then, for $\wb{\theta}$ a minimizer we have
\begin{align*}
\| \wb{\theta} - \theta_h \|_{A}^2 = \frac{(\Delta(\pi) + \epsilon)^2}{\|\Psi_h^\pi - \Psi_h^{\star}\|_{A^{-1}}^2}.
\end{align*}
\end{lemma}
\begin{proof}
To simplify notation, let us define $b:= \sum_{l\in[H], l\neq h}\langle \theta_{l}, \Psi_l^\pi - \Psi_l^{\star}\rangle $. The corresponding Lagrange dual problem is
\begin{align*}
    \max_{\lambda \geq 0} \min_{\theta \in \mathbb{R}^d} \left\{ \| \theta - \theta_h \|_{A}^2 - \lambda \left( \langle \theta, \Psi_h^\pi - \Psi_h^{\star}\rangle + b - \epsilon \right) \right\}.
\end{align*}
Let $f(\theta,\lambda)$ denote the resulting objective function. Taking the gradient w.r.t. $\theta$,
\begin{align*}
    \nabla_{\theta}f(\theta,\lambda) = 2A(\theta-\theta_h) - \lambda(\Psi_h^\pi - \Psi_h^{\star}),
\end{align*}
and equating it to zero, we obtain
\begin{align*}
    \theta = \theta_h + \frac{\lambda}{2}A^{-1}(\Psi_h^\pi - \Psi_h^{\star}).
\end{align*}
Plugging this back to the original objective we get
\begin{align*}
    f(\lambda) &= \frac{\lambda^2}{4}\|A^{-1}(\Psi_h^\pi - \Psi_h^{\star})\|_{A}^2 - \lambda \left(\langle \theta_h, \Psi_h^\pi - \Psi_h^{\star}\rangle + \frac{\lambda}{2}\|\Psi_h^\pi - \Psi_h^{\star}\|_{A^{-1}}^2 + b - \epsilon \right)\\ &= -\frac{\lambda^2}{4}\|\Psi_h^\pi - \Psi_h^{\star}\|_{A^{-1}}^2  - \lambda \left(\langle \theta_h, \Psi_h^\pi - \Psi_h^{\star}\rangle + \sum_{l\in[H], l\neq h}\langle \theta_{l}, \Psi_l^\pi - \Psi_l^{\pi^\star}\rangle - \epsilon \right)\\ &= -\frac{\lambda^2}{4}\|\Psi_h^\pi - \Psi_h^{\star}\|_{A^{-1}}^2  + \lambda \left(\Delta(\pi)+\epsilon\right).
\end{align*}
Differentiating with respect to $\lambda$ and equating to zero we obtain
\begin{align*}
    \lambda = \frac{ 2\left(\Delta(\pi)+\epsilon\right)}{\|\Psi_h^\pi - \Psi_h^{\star}\|_{A^{-1}}^2}.
\end{align*}
Therefore, plugging this back into the objective value
\begin{align*}
    \| \wb{\theta} - \theta_h \|_{A}^2 = \frac{ \left(\Delta(\pi)+\epsilon\right)^2}{\|\Psi_h^\pi - \Psi_h^{\star}\|_{A^{-1}}^2}.
\end{align*}
\end{proof}

\begin{lemma}[Bretagnolle–Huber inequality, see, e.g., Thm.~14.2 of \cite{lattimore2020bandit}]\label{lemma:bh-ineq}
Let $\mathbb{P}$ and $\mathbb{Q}$ be probability
measures on the same measurable space $(\Omega, \mathcal{F})$ and let $E\in\mathcal{F}$ be an arbitrary
event. Then,
\begin{align*}
\mathbb{P}(E) + \mathbb{Q}(E^c) \geq \frac{1}{2}e^{-\mathrm{KL}(\mathbb{P},\mathbb{Q})}.
\end{align*}
\end{lemma}


\section{\hls{} is Sufficient: Proofs of Section \ref{sec:sufficient}}\label{app:sufficient}
We first prove that \hls{} is sufficient for a whole class of algorithms, as done in the proof sketch of Section~\ref{sec:sufficient}. We will then instantiate this result to ELEANOR and LSVI-UCB.

Consider the following assumptions.
\begin{assumption}\label{asm:algo.weak}
Consider a feature map $\{\phi_h\}_{h\in[H]}$ and a Q-function estimate $\overline{Q}_h^k$.     
    There is an event $G(\delta)$ that holds with probability at least $1-\delta$ under which: 
    \begin{enumerate}[label={(\alph*)}]
    \item
    Global optimism: $\overline{V}_1^k(s) \ge V_1^\star(s)$ where $\overline{V}_h^k(s)=\max_{a\in\Aspace}\{\overline{Q}_h^k(s,a)\}$,\label{asm:algo.globalopt}
    \item\label{asm:algo.confidence}
    Confidence set: let $\Lambda_h^{k}=\sum_{i=1}^{k-1}\phi_h(s_h^i,a_h^i)\phi_h(s_h^i,a_h^i)^\transp + \lambda I$ and $\beta_k \in \mathbb{R}_+$ be increasing and logarithmic in $k$, then
    $\overline{V}_h^k(s_h^k) - V_h^{\pi^k}(s_h^k) \le 2\beta_k\norm{\phi_h(s_h^k,a_h^k)}_{(\Lambda_h^k)^{-1}} + \EV_{s'\sim p_h(s_h^k,a_h^k)}\left[\overline{V}_{h+1}^k(s')-V^{\pi^k}_{h+1}(s')\right]$,
\end{enumerate}
simultaneously for all $h\in[H]$, $k\ge 1$ and $s\in\Sspace$, where $\delta\in(0,1)$ is a parameter of the algorithm.
\end{assumption}

\begin{assumption}\label{asm:algo.strong} 
    The algorithm satisfies Assumption~\ref{asm:algo.weak}, and additionally
    there exist a set of constants $(C_h)_{h\in[H]}$ such that, under the event $G(\delta)$: \begin{enumerate}[label={(\alph*)}]\addtocounter{enumi}{2}
    \item\label{asm:algo.localopt}
    (Almost) local optimism:
    $\qquad\overline{Q}^k_h(s,a)+C_h\beta_k\norm{\phi_h(s,a)}_{(\Lambda_h^k)^{-1}} \ge Q^\star_h(s,a)$,
    \end{enumerate}
    for all $h=2,\dots,H$, $k\ge 1$, $s\in\Sspace$ and $a\in\Aspace$.
\end{assumption}

Assumption~\ref{asm:algo.strong} characterizes the class of algorithms for which we are going to prove a constant bound on the regret under \hls{}. However, we first study the regret under the weaker Assumption~\ref{asm:algo.weak}, following the proof pattern from~\citep{Jin2020linear}.


\begin{lemma}\label{lem:general.worstcase}
    Under Assumption~\ref{asm:algo.weak}, assuming event $G(\delta)$ holds, there exists a $\wt{O}(\sqrt{K})$ function $g$ such that, with probability $1-\delta$, for all $K\ge 1$:
\begin{equation}
    R(K) \le H\beta_K\sqrt{2dK\log(1+K/\lambda)} + 2H^2\sqrt{K\log(2HK/\delta)} = \wt{O}(\sqrt{K}).
\end{equation}

\end{lemma}
\begin{proof}
Under event $G(\delta)$:
\begin{align}
    R(K) &= \sum_{k=1}^K V_1^\star(s_1^k) - V_1^{\pi^k}(s_1^k) \nonumber\\
    &\le \sum_{k=1}^K \overline{V}_1^k(s_1^k) - V_1^{\pi^k}(s_1^k) &&\text{\ref{asm:algo.globalopt}}\\
    &\le \underbrace{2\sum_{h=1}^H\beta_{K}\sum_{k=1}^K\norm{\phi_h(s_h^k,a_h^k)}_{(\Lambda_{h}^k)^{-1}}}_{(A)} + \underbrace{\sum_{k=1}^K\sum_{h=1}^H\zeta_{h}^k}_{(B)},
\end{align} 
where the last inequality is from recursive application of~\ref{asm:algo.confidence} and the fact that $\beta_k$ is increasing, and:
\begin{equation}
    \zeta_{h}^k = \EV_{s'\sim p_h(s_h^k,a_h^k)}[\overline{V}_{h+1}^k(s') - V^{\pi^k}_{h+1}(s')] - \overline{V}_{h+1}^k(s_{h+1}^k) + V^{\pi^k}_{h+1}(s_{h+1}^k),
\end{equation}
where expectations are conditioned on the history up to the beginning of episode $k$. We bound $(A)$ using the Elliptical Potential Lemma~\citep[e.g.,][]{abbasi2011improved}:
\begin{align}
    (A)&=2\beta_K\sum_{h=1}^H\sum_{k=1}^K\norm{\phi_h(s_h^k,a_h^k)}_{(\Lambda_{h}^k)^{-1}}\\ 
    &2\beta_K\sum_{h=1}^H\le \sqrt{K\sum_{k=1}^K\norm{\phi_h(s_h^k,a_h^k)}^2_{(\Lambda_{h}^k)^{-1}}} \\
    &\le H\beta_K\sqrt{2dK\log(1+K/\lambda)}.
\end{align}
Since $\zeta_h^k$ is a martingale difference sequence with $\zeta_h^k\le 2H$, we can use Azuma's inequality (Prop.~\ref{prop:azuma}) to bound $(B)$:
\begin{equation}
    \sum_{k=1}^K\zeta_h^k \le 2H\sqrt{K \log(2K/\delta_{h})}, 
\end{equation}
with probability $1-\delta_h$ for all $K\ge 1$.
To make it hold with probability $1-\delta$ for all $h\in[H]$, we set $\delta_{h}=\delta/H$. Finally:
\begin{equation}
    (B) = \sum_{h=1}^H\sum_{k=1}^K\zeta_h^k \le 2H^2\sqrt{K\log(2HK/\delta)}.
\end{equation}
\end{proof}

The stronger Assumption~\ref{asm:algo.strong} is needed to upper-bound the gaps.
\begin{lemma}\label{lem:gapbound}
    Under Assumption~\ref{asm:algo.strong}, assuming event $G(\delta)$ holds, for all $s\in\Sspace$, $h\in[H]$ and $k\ge 1$:
    \begin{equation*}
        \Delta_h(s,\pi_h^k(s)) \le 2\EV_{\pi^k}\left[\sum_{i=h}^H\beta_k\norm{\phi_i(s_i,a_i)}_{(\Lambda_i^k)^{-1}}\Bigg| s_h=s\right] + \indi{h>1}C_h\beta_k\norm{\phi^\star(s)}_{(\Lambda_h^k)^{-1}}.
    \end{equation*}
\end{lemma}
\begin{proof}
    \begin{align}
        \Delta_h(s,\pi_h^k(s))&= V^\star_h(s) - Q^\star_h(s_h^k,\pi_h^k(s)) \\
        &\le V^\star_h(s) - Q^{\pi^k}_h(s_h^k,\pi_h^k(s)) \\
        &= V^\star_h(s) - V^{\pi^k}_h(s) \\
        &= Q^\star_h(s, \pi^\star_h(s)) - V^{\pi^k}_h(s) \\
        &\le \overline{Q}_h^k(s,\pi^\star_h(s)) + \indi{h>1}C_h\beta_k\norm{\phi_h(s,\pi^\star_h(s))}_{(\Lambda_{h}^k)^{-1}} - V^{\pi^k}_h(s) \label{pp:gapbound.1} \\
        &\le \overline{V}_h^k(s) + \indi{h>1}C\sqrt{\gamma_{hk}}\norm{\phi_h(s,\pi^\star_h(s))}_{\Lambda_{hk}^{-1}} - V^{\pi^k}_h(s) \\
        &\le 2\EV_{\pi^k}\left[
            \sum_{i=h}^H\beta_k\norm{\phi_i(s_i,a_i)}_{(\Lambda_i^k)^{-1}}
            \Bigg| s_h=s\right]  
            \\\nonumber&\qquad
            +\indi{h>1}C_h\beta_k\norm{\phi_h(s_h^k,\pi^\star_h(s_h^k))}_{(\Lambda_h^k)^{-1}},
    \end{align}
    where~\eqref{pp:gapbound.1} uses~\ref{asm:algo.globalopt} for $h=1$ and~\ref{asm:algo.localopt} for $h>1$, while the last inequality is from recursive application of~\ref{asm:algo.confidence}.
\end{proof}

Now we can prove our main result on constant regret:

\begin{theorem}\label{thm:general.const}
Any algorithm satisfying Assumption~\ref{asm:algo.strong} enjoys constant regret if the representation has the \hls property (Asm.~\ref{asm:feature.structure}) and Assumption~\ref{asm:gap} on the minimum gap holds. In general, let $g:\mathbb{N}\to\Reals_+$ be any increasing $\wt{O}(\sqrt{K})$ function such that, with probability $1-2\delta$ for all $K\ge 1$, $R(K)\le g(K)$. Then, under Assumptions~\ref{asm:gap},~\ref{asm:feature.structure},~\ref{asm:algo.strong}, with probability $1-3\delta$ for all $K\ge 1$:
\begin{equation}
    R(K) \le g(\overline{\kappa}) = O(1),
\end{equation}
where $\overline{\kappa}$ is a constant independent of $K$. 
\end{theorem}
\begin{proof}
    First notice that a valid regret upper bound $g(K)$ always exists due to Lemma~\ref{lem:general.worstcase}. 
    Moreover, due to Asm.~\ref{asm:feature.structure}, for all $h\in[H]$ and $k\ge1$, we have $\phi_h(s,\pi_h^k(s))\in\mathrm{span}\{\phi^\star_h(s)|\rho_h^\star(s)>0\}$ for all $s\in\Sspace$ such that $\rho_h^{\pi^k}(s)>0$.
    Hence, with probability $1-2\delta$, the requirements of Lemma~\ref{lem:key} are satisfied  and we can apply it to the gap upper bound from Lemma~\ref{lem:gapbound}. So, with probability $1-3\delta$, for all $s\in\Sspace$, $h\in[H]$ and $k\ge \wt{\kappa} = \max_{h\in[H]}\wt{\kappa}_h$:
    \begin{align}
        \Delta_h(s,\pi_h^k(s)) &\le 2\EV_{\pi^k}\left[\sum_{i=h}^H\beta_k\norm{\phi_i(s_i,a_i)}_{(\Lambda_i^k)^{-1}}\Bigg| s_h=s\right] 
        \nonumber\\&\qquad+ \indi{h>1}C_h\beta_k\norm{\phi^\star(s)}_{(\Lambda_h^k)^{-1}}\\
        &\le (2+\indi{h>1}C_h)\beta_k\sum_{i=h}^H\frac{k+\lambda-g(k)-8\sqrt{k\log(2dHk/\delta)}}{(k\lambda_{i}^{+}+\lambda - g(k)-8\sqrt{k\log(2dHk/\delta)})^{3/2}}.\label{eq:kappabound}
    \end{align}
    Assume for now that $k\ge\wt{\kappa}$. From the previous inequality, since $g(k)=\wt{O}(\sqrt{k})$ and $\beta_k=\wt{O}(1)$, there exists a $\kappa_h$ independent of $K$ such that, for $k>\kappa_h$:
    \begin{equation}\label{eq:deltaineq}
        \Delta_h(s,\pi_h^k(s)) \le \Delta_{\min}.
    \end{equation}
    Under Asm.~\ref{asm:gap}, this implies $\Delta_h(s,\pi_h^k(s))=0$. Let $\overline{\kappa} = \max\{\wt{\kappa},\max_h\{\kappa_h\}\}$. For $k>\overline{\kappa}$, all the gaps are zero. Finally, by Prop.~\ref{prop:sumofgaps}:
    \begin{align}
        R(K) &=\sum_{k=1}^K\EV_{\pi^k}\left[\sum_{h=1}^H\Delta_h(s_h,a_h)\Bigg|s_1=s_1^k\right] \\
        &=\sum_{k=1}^{\overline{\kappa}}\EV_{\pi^k}\left[\sum_{h=1}^H\Delta_h(s_h,a_h)\Bigg|s_1=s_1^k\right] + \sum_{k=\overline{\kappa}+1}^K\EV_{\pi^k}\left[\sum_{h=1}^H\underbrace{\Delta_h(s_h,a_h)}_{=0}\Bigg|s_1=s_1^k\right] \\
        &= R(\overline{\kappa}) \le g(\overline{\kappa}).
    \end{align}
\end{proof}

Finally, we instantiate the general result of~\ref{thm:general.const} to ELEANOR on MDPs with Bellman closure and LSVI-UCB on low-rank MDPs, by showing that they satisfy Assumption~\ref{asm:algo.strong}. 

\paragraph{Proof of Theorem~\ref{thm:const.ele}.}
\begin{proof}[\normalfont Let:]
\begin{equation}\label{eq:beta.eleanor}
    \beta_{k} =  H\sqrt{\frac{d}{2}\log(1+k/d)+d\log(1+4\sqrt{dk}) +\log\frac{2Hk^2}{\delta}} + 1,
\end{equation}
and define event $G(\delta)$ as in Lemma 2 from~\citep{Zanette2020low}.
We have~\ref{asm:algo.globalopt} by Lemma 7 from~\citep{Zanette2020low}, while \ref{asm:algo.confidence} can be extracted from the proof of Theorem 1 from~\citep{Zanette2020low}. To prove~\ref{asm:algo.localopt}, we use the fact that the MDP satisfies Bellman closure, hence there exist $\vtheta_1^\star,\dots,\vtheta_H^\star$ such that~\citep[Lemma 6 from][]{Zanette2020low}:
\begin{equation}
    Q_h^\star(s,a) = \phi_h(s,a)^\transp\vtheta_h^\star.
\end{equation}
By Lemma 7 from~\citep{Zanette2020low}, $\vtheta_1^\star,\dots,\vtheta_H^\star$ is a feasible solution for $\overline{\vtheta}_1,\dots,\overline{\vtheta}_H$ in ELEANOR's program~\citep[Definition 2 from][]{Zanette2020low}. Due to the program's constraints:
\begin{equation}
        \norm{\vtheta^\star_h - \widehat{\vtheta}_h^k}_{\Lambda_{h}^k} \le \beta_k.
    \end{equation}
    Let $\overline{\vtheta}_1^k,\dots,\overline{\vtheta}_H^k$ be the values that are actually selected by ELEANOR's program. Since they are subject to the same constraints, by the triangular inequality:
    \begin{equation}
        \norm{\vtheta^\star_h - \overline{\vtheta}_h^k}_{\Lambda_{h}^k} \le 2\beta_k.
    \end{equation}
    Finally, since $\overline{Q}_h^k(s,a)=\phi_h(s,a)^\transp\overline{\vtheta}_h^k$:
    \begin{align}
        Q_h^\star(s_h,a_h) &= \phi_h(s_h,a_h)^\transp\vtheta^\star_h\\
        &= \phi_h(s_h,a_h)^\transp\overline{\vtheta}_h^k + \phi_h(s_h,a_h)^\transp(\vtheta^\star_h-\overline{\vtheta}_h^k) \\
        &\le\overline{Q}_h^k(s_h,a_h) + \norm{\phi(s_h,a_h)}_{(\Lambda_{h}^k)^{-1}}\norm{\vtheta^\star_h-\overline{\vtheta}_h^k}_{\Lambda_{h}^k} \\
        &\le \overline{Q}_h^k(s_h,a_h) + 2\beta_k\norm{\phi(s_h,a_h)}_{(\Lambda_{h}^k)^{-1}},
    \end{align}
    so~\ref{asm:algo.localopt} holds with $C_h=2$. So Asm.~\ref{asm:algo.strong} holds and we can invoke Theorem~\ref{thm:general.const} with the upper bound $g$ from Lemma~\ref{lem:general.worstcase} and the $\beta_k$ given above to obtain:
    \begin{align}
        R(K) &\le H^2\left(\sqrt{\frac{d}{2}\log(1+\overline{\kappa}/d)+d\log(1+4\sqrt{d\overline{\kappa}}) + \log(H\overline{\kappa}^2) +\log\frac{2}{\delta}} + H\right)
        \nonumber\\&\qquad\times
        \sqrt{2d\overline{\kappa}\log(1+\overline{\kappa}/\lambda)} + 2H^2\sqrt{\overline{\kappa}\log(2H\overline{\kappa}/\delta)} \\
        &\lesssim H^{3/2}d\sqrt{\wb\tau\log\frac{\wb\tau}{\delta}},
    \end{align}
    where $\overline{\tau}=H\overline{\kappa}$.
\end{proof}

\begin{remark}
We have slightly modified the ELEANOR algorithm to obtain any-time regret bounds. In particular, we have replaced the fixed $\delta'=\delta/(2T)$ term in the original $\beta_k$ (see the proof of Lemma 2 in~\citep{Zanette2020low}) with the adaptive $\delta/(2Hk^2)$. This still makes event $G(\delta)$ hold with probability $1-\delta$, but without knowledge of the horizon $K$.
This only affects logarithmic terms.
Also notice that we have considered the case of zero inherent Bellman error ($\mathcal{I}=0$), which corresponds to Bellman closure, and we have taken $[0,H]$, not $[0,1]$, as the range of the value function (see the comment following Theorem 1 in~\citep{Zanette2020low}). 
\end{remark}

For LSVI-UCB, we can instantiate Theorem~\ref{thm:general.const} with the problem-dependent logarithmic lower bound by~\citet{He2020loglin} in place of the worst-case upper bound from Lemma~\ref{lem:general.worstcase}.

\paragraph{Proof of Theorem~\ref{thm:const.lsvi}.}
\begin{proof}[\normalfont Let:]
\begin{equation}\label{eq:beta.lsvi}
    \beta_k=c_\beta dH\sqrt{\log(2dHk/\delta)},
\end{equation}
where $c_\beta$ is a constant defined in Lemma C.3 from~\citep{Jin2020linear}, and define event $G(\delta)$ as in Lemma B.3 from~\citep{Jin2020linear}. Then since the MDP is low-rank, by Lemma B.5 from~\citep{Jin2020linear} we have both~\ref{asm:algo.globalopt} and~\ref{asm:algo.localopt} with $C_h=0$. We get~\ref{asm:algo.confidence} by Lemma B.4 from~\cite{Jin2020linear}. So Asm.~\ref{asm:algo.strong} holds and, under Asm~\ref{asm:gap}, we can instantiate Theorem~\ref{thm:general.const} with the logarithmic regret bound from Theorem 4.4 by~\citet{He2020loglin}:
\begin{equation}
    g(k) = 9HG(k)\log G(k) + \frac{16H^2}{3}\log\frac{\log\lceil Hk\rceil}{\delta} + 2,
\end{equation}
where:
\begin{equation}
    G(k)\propto\frac{d^3H^4\log(4dH^2k(k+1)\log(H/\Delta_{\min})/\delta)}{\Delta_{\min}}.
\end{equation}
So:
\begin{equation}
    R(K) \le g(\overline{\kappa})\simeq \frac{d^3 H^5}{\Delta_{\min}} \log \big(dH^2 \wb\kappa / \delta\big).
\end{equation}
\end{proof}

\begin{remark}
We have slightly modified the LSVI-UCB algorithm to obtain any-time regret bounds. In particular, we have replaced the fixed $\iota=\log(2dT/\delta)$ term in the original $\beta_k$ (see Theorem 3.1 from~\citep{Jin2020linear}) with the adaptive $\log(4dHk^2/\delta)$. This still makes event $G(\delta)$ hold with probability $1-\delta$, but without knowledge of the horizon $K$.
We have also re-written the logarithmic regret bound by~\citet{He2020loglin} (Theorem 4.4) to hold with probability $1-2\delta$. These changes only affect logarithmic terms.
\end{remark}


\begin{lemma}\label{lem:kappabound_lsvi}
The critical time $\overline{\kappa}$ from Theorem~\ref{thm:const.lsvi} for LSVI-UCB is upper bounded as:
\begin{equation}
    \overline{\kappa} \le \max\left\{
    \frac{48c_1^2H^4d^3}{\lambda_+^2}\log\left(\frac{32c_1^2H^5d^4}{\lambda_+^2\delta}\right),
    \frac{432c_2^2H^4d^2}{\Delta_{\min}^2\lambda_+^3}\log\left(\frac{288d^3H^5c_2^2}{\Delta_{\min}^2\lambda_+^3\delta}\right)
    \right\}
\end{equation}
where $\lambda_+ = \min_{h\in[H]}\{\lambda_h^+\}$ and $c_1,c_2$ are universal constants.
\end{lemma}
\begin{proof}
For LSVI-UCB we have (see the proof of Theorem~\ref{thm:const.lsvi}):
\begin{align}
    &g(k) \le c_1 H^2d^{3/2}\sqrt{k\log(2dHk/\delta)},\\
    &\beta_k = c_2 d H \sqrt{\log(2dHk/\delta)},
\end{align}
for some universal constants $c_1,c_2$. We assume $\lambda=1$ and $c_1\ge 8$.

We will use the fact that a sufficient condition for $k\ge a\log(bk)$ is $k\ge3a\log(ab)$ for $k>0$ and reasonable values of the constants $a,b$. See App. C.6 from~\cite{papini2021leveraging} for details. This immediately implies that a sufficient condition for $k\ge a\sqrt{k\log(bk)}$ is:
\begin{equation}\label{eq:lambert}
    k\ge 3a^2\log(a^2b)
\end{equation}
We divide the rest of the proof in three parts:
\paragraph{Part 1.} First, $\kappa$ must satisfy the invertibility conditions from Lemma~\ref{lem:key}. To make matrix $B_h^k=k\Lambda_h^\star+\lambda I - g(k) + 8\sqrt{k\log(2dHk/\delta)}$ invertible for each $h$, we first require the positive eigenvalues of $\Lambda_h^\star$ to map into positive eigenvalues of $B_h^k$. A sufficient condition for this is:
\begin{align}
    &k\lambda_+ > 1 + g(k) + 8\sqrt{k\log(2dHk/\delta)}\\
    &k \ge \frac{c_1H^2d^{3/2}+8}{\lambda_+}\sqrt{k\log(2dHk/\delta)}\\
    &k \ge \frac{2c_1H^2d^{3/2}}{\lambda_+}\sqrt{k\log(2dHk/\delta)}\\
    &k \ge \frac{12c_1^2H^4d^{3}}{\lambda_+^2}\log\left(\frac{8c_1^2H^5d^4}{\lambda_+^2\delta}\right) \triangleq \overline{\kappa}_1,
\end{align}
where the latter is from~\eqref{eq:lambert}. We also need the zero eigenvalues of $\Lambda_h^\star$ to map into negative eigenvalues of $B_h^k$. However, this just requires $\lambda - g(k)+ 8\sqrt{k\log(2dHk/\delta)} < 0$ which is already true for $k=1$ given $\lambda=1$.

\paragraph{Part 2.} We require $\overline{\kappa}$ to satisfy the following, which will make the analysis of Part 3 easier:
\begin{equation}\label{eq:easier}
    g(k) + 8\sqrt{k\log(2dHk/\delta)} \le \frac{k\lambda_+}{2}.
\end{equation}
\end{proof}
After rearranging, we can proceed precisely as in Part 1, only with different numerical constants, obtaining:
\begin{equation}
    k\ge \frac{48c_1^2H^4d^{3}}{\lambda_+^2}\log\left(\frac{32c_1^2H^5d^4}{\lambda_+^2\delta}\right) \triangleq \overline{\kappa}_2.
\end{equation}

\paragraph{Part 3.} Assume for now that $k\ge\overline{\kappa}_2$. Since $\overline{\kappa}_2\ge \overline{\kappa}_1$, the invertibility conditions from Lemma~\ref{lem:key} are satisfied and, by the proof of Theorem~\ref{thm:general.const}, regret is zero for all time $k$ such that:
\begin{equation}
    (2+\indi{h>1}C_h)\beta_k\sum_{i=h}^H\frac{k+\lambda-g(k)-8\sqrt{k\log(2dHk/\delta)}}{(k\lambda_{i}^{+}+\lambda - g(k)-8\sqrt{k\log(2dHk/\delta)})^{3/2}} \le \Delta_{\min},
\end{equation}
for all $h$. Using the definition of $\lambda^+$, $\lambda=1$ and $C_h=0$ for LSVI-UCB, a sufficient condition is:
\begin{align}
    &2H\beta_k\frac{k+1-g(k)-8\sqrt{k\log(2dHk/\delta)}}{(k\lambda_{+}+1 - g(k)-8\sqrt{k\log(2dHk/\delta)})^{3/2}} \le \Delta_{\min},\\
    &2H\beta_k\frac{2k}{(k\lambda_{+} - g(k)-8\sqrt{k\log(2dHk/\delta)})^{3/2}} \le \Delta_{\min}.
\end{align}
Since $k\ge\overline{\kappa}_2$, by~\eqref{eq:easier}, we just need:
\begin{align}
    &2H\beta_k\frac{2k}{\left(\frac{1}{2}k\lambda_{+}\right)^{3/2}} \le \Delta_{\min}.
\end{align}
Rearranging and using the definition of $\beta_k$:
\begin{align}
    &\sqrt{k} \ge \frac{12c_2H^2d}{\Delta_{\min}\lambda_+^{3/2}}\sqrt{\log(2dHk/\delta)}\\
    &k \ge \frac{12c_2H^2d}{\Delta_{\min}\lambda_+^{3/2}}\sqrt{k\log(2dHk/\delta)},
\end{align}
and again from~\eqref{eq:lambert}:
\begin{equation}
    k \ge \frac{432c_2^2H^4d^2}{\Delta_{\min}^2\lambda_+^3}\log\left(\frac{288c_2^2d^3H^5}{\Delta_{\min}^2\lambda_+^2\delta}\right) \triangleq \overline{\kappa}_3.
\end{equation}

The proof is concluded by taking $\overline{\kappa} = \max\{\overline{\kappa}_2,\overline{\kappa}_3\}$.

\begin{lemma}\label{lem:kappabound_ele}
The critical time $\overline{\kappa}$ from Theorem~\ref{thm:const.ele} for ELEANOR is upper bounded as:
\begin{equation}
    \overline{\kappa} \le \max\left\{
    \frac{48c_1^2H^4d^2}{\lambda_+^2}\log\left(\frac{32c_1^2H^5d^3}{\lambda_+^2\delta}\right),
    \frac{432c_2^2H^4d}{\Delta_{\min}^2\lambda_+^3}\log\left(\frac{288d^2H^5c_2^2}{\Delta_{\min}^2\lambda_+^3\delta}\right)
    \right\}
\end{equation}
where $\lambda_+ = \min_{h\in[H]}\{\lambda_h^+\}$ and $c_1,c_2$ are universal constants.
\end{lemma}
\begin{proof}
The proof is the same as for Lemma~\ref{lem:kappabound_lsvi}, except that for ELEANOR we have (see the proof of Theorem~\ref{thm:const.ele}):
\begin{align}
    &g(k) \le c_1 H^2d\sqrt{k\log(2dHk/\delta)}\\
    &\beta_k \le c_2 H\sqrt{d\log(2dHk/\delta)},
\end{align}
where $c_1,c_2$ are universal constants. The three critical times are then:
\begin{align}
    &\overline{\kappa}_1 = \frac{12c_1^2H^4d^2}{\lambda_+^2}\log\left(\frac{8c_1^2H^5d^3}{\lambda_+^2\delta}\right)\\
    &\overline{\kappa}_2 = \frac{48c_1^2H^4d^{2}}{\lambda_+^2}\log\left(\frac{32c_1^2H^5d^3}{\lambda_+^2\delta}\right)\ge \overline{\kappa}_1\\
    &\overline{\kappa}_3 = \frac{432c_2^2H^4d}{\Delta_{\min}^2\lambda_+^3}\log\left(\frac{288c_2^2d^2H^5}{\Delta_{\min}^2\lambda_+^2\delta}\right),
\end{align}
and we can take $\overline{\kappa}=\max\{\overline{\kappa}_2,\overline{\kappa}_3\}$.
\end{proof}


\section{Representation Selection: Proofs of Section \ref{sec:rep.selection}}\label{app:rep.selection}
The main ingredient behind the proofs of Theorems~\ref{thm:mix.stage} and~\ref{thm:mix.state} In order to show a regret guarantee for the LSVI-LEADER algorithm, we start by showing a  version of Lemma B.4 in~\citep{Jin2020linear} that takes into account the presence of multiple representations. 

First we need the corresponding version of Lemma D.6 in~\citep{Jin2020linear}. 

\begin{lemma}
Given an MDP $M$ and a set of representations $\{\Phi_j\}_{j\in[N]}$ satisfying the low-rank assumption (Asm.~\ref{asm:lowrank}). Let $\mathcal{V}$ denote a class of functions mapping from $\mathcal{S}$ to $\mathbb{R}$ with the following parametric form,
\begin{equation*}
    V(\cdot) = \min\left(  \min_{j \in [N]}\max_a \boldsymbol{w}_j^\top \phi_j(\cdot, a ) + \beta \sqrt{ \phi_j(\cdot, a)^\top \boldsymbol{\Lambda}_j^{-1} \phi_j(\cdot, a) }  , H\right) 
\end{equation*}

where the parameters $\{\boldsymbol{w}_j, \boldsymbol{\Lambda}_j \}_{j=1}^N, \beta $ satisfy $\| \boldsymbol{w}\| \leq L$, $\beta \in [0, B]$ and the minimum eigenvalue of $\boldsymbol{\Lambda}_j$ satisfies $\lambda_{\mathrm{min}}(\boldsymbol{\Lambda}_j ) \geq \lambda$. Assume $\| \boldsymbol{\phi}(s,a) \| \leq 1$ for all $(s,a)$ pairs and let $\mathcal{N}_\epsilon$ be the $\epsilon-$covering number of $\mathcal{V}$ with respect to the distance $\mathrm{dist}(V, V') = \sup_s |V(s) - V'(s)|$. Then,
\begin{equation*}
\log \mathcal{N}_\epsilon \leq N\left(d \log(1+4L/\epsilon) + d^2 \log\left(1+8d^{1/2}B^2/(\lambda \epsilon^2) \right)     \right)
\end{equation*}

\end{lemma}

\begin{proof}
Let's reparametrize the function class $\mathcal{V}$ by $\mathbf{A}_j = \beta^2 \boldsymbol{\Lambda}_j^{-1}$, so we have,
\begin{equation}\label{equation::second_characterization_of_V}
    V(\cdot) = \min\left(  \min_{j \in [N]}\max_a \boldsymbol{w}_j^\top \phi_j(\cdot, a ) +  \sqrt{ \phi_j(\cdot, a)^\top \boldsymbol{A}_j \phi_j(\cdot, a) }  , H\right) 
\end{equation}
for $\| \boldsymbol{w}_j\|\leq L$ and $\| \boldsymbol{A}_j \| \leq B^2 \lambda^{-1}$. For any two functions $V_1, V_2 \in \mathcal{V}$, let them take the form in Equation~\ref{equation::second_characterization_of_V} with parameters $(\{\boldsymbol{w}_j^{(1)}, \boldsymbol{A}_j^{(1)}\}_{j=1}^N$ and $(\{\boldsymbol{w}_j^{(2)}, \boldsymbol{A}_j^{(2)}\}_{j=1}^N$. Then since $\min_j$, $\min(\cdot, H)$ and $\max_a$ are contraction maps, we have

\begin{align}
    \mathrm{dist}(V_1, V_2) &\leq  \sup_{j,s,a}\Big|   \left[   \left(\boldsymbol{w}^{(1)}_j\right)^\top \phi_j(\cdot, a ) +  \sqrt{ \phi_j(\cdot, a)^\top \boldsymbol{A}^{(1)}_j \phi_j(\cdot, a) }    \right] - \\
    &\quad\left[ \left(\boldsymbol{w}^{(2)}_j\right)^\top \phi_j(\cdot, a ) +  \sqrt{ \phi_j(\cdot, a)^\top \boldsymbol{A}^{(2)}_j \phi_j(\cdot, a) }   \right]         \Big| \notag\\
    &\leq \sup_j\left( \sup_{\| \phi_j\|\leq 1}  \left|   \left[   \left(\boldsymbol{w}^{(1)}_j\right)^\top \phi_j +  \sqrt{ \phi_j^\top \boldsymbol{A}^{(1)}_j \phi_j }    \right] - \left[ \left(\boldsymbol{w}^{(2)}_j\right)^\top \phi_j +  \sqrt{ \phi_j^\top \boldsymbol{A}^{(2)}_j \phi_j }   \right]         \right| \right)\notag \\
    &\leq \sup_j \left(\sup_{\| \phi_j\|\leq 1} \left| \left(  \boldsymbol{w}^{(1)}_j - \boldsymbol{w}^{(2)}_j   \right)^\top \phi_j \right| + \sup_{\| \phi_j\|\leq 1}  \sqrt{\left| \phi_j^\top \left(  \boldsymbol{A}^{(1)}_j - \boldsymbol{A}^{(2)}_j   \right)\phi_j \right|  }    \right) \notag\\
    &= \sup_j \| \boldsymbol{w}^{(1)}_j - \boldsymbol{w}^{(2)}_j  \| + \sqrt{\|    \boldsymbol{A}^{(1)}_j -  \boldsymbol{A}^{(2)}_j    \|  } \notag \\
    &\leq \sup_j \| \boldsymbol{w}^{(1)}_j - \boldsymbol{w}^{(2)}_j  \| + \sqrt{\|    \boldsymbol{A}^{(1)}_j -  \boldsymbol{A}^{(2)}_j    \|_F  } \label{equation::upper_bound_on_norm}
\end{align}

For matrices $\| \cdot \|$ and $\| \cdot \|_F$ denote the matrix operator norm and the frobenius norm respectively. 

Let $\mathcal{C}_j^{\boldsymbol{w}}$ be an $\epsilon/2$ cover of $\{ \boldsymbol{w}_j \in \mathbb{R}^d | \| \boldsymbol{w}_j \|\leq L\}$ with respect to the $2$-norm and let $\mathcal{C}_j^{\boldsymbol{A}}$ be an $\epsilon^2/4-$cover of $\{ \boldsymbol{A} \in \mathbb{R}^{d\times d} | \| \boldsymbol{A}\|_F \leq d^{1/2} B^2 \lambda^{-1} \}$ with respect to the Frobenius norm. By Lemma D.5. in~\citep{Jin2020linear} we know that,

\begin{equation*}
    | \mathcal{C}^{\boldsymbol{w}}_j | \leq (1+4L/\epsilon)^d, \qquad | \mathcal{C}_j^{\boldsymbol{A}} | \leq \left(1+8d^{1/2} B^2 / (\lambda \epsilon^2) \right)^{d^2}
\end{equation*}

By Equation~\ref{equation::upper_bound_on_norm}, for any $V_1 \in \mathcal{V}$ there exists points $\{ \boldsymbol{w}_j^{(2)}\}_{j=1}^N$ and $\{ \boldsymbol{A}_j^{(2)}\}_{j=1}^N$ such that $V_2$ parametrized by $( \{ \boldsymbol{w}_j^{(2)}\}_{j=1}^N,  \boldsymbol{A}_j^{(2)}\}_{j=1}^N   )$ satisfies $\mathrm{dist}(V_1, V_2) \leq \epsilon$. Hence it holds that $\mathcal{N}_\epsilon \leq \left(|  \mathcal{C}^{\boldsymbol{w}}_j| |\mathcal{C}_j^{\boldsymbol{A}}  | \right)^N$, which gives:

\begin{equation*}
  \log  \mathcal{N}_\epsilon \leq  N\left(d \log(1+4L/\epsilon) + d^2 \log\left(1+8d^{1/2}B^2/(\lambda \epsilon^2) \right)     \right).
\end{equation*}
\end{proof}

\begin{lemma}[Multi-representation version of Lemma B.3 in~\citep{Jin2020linear}]\label{lemma::representation_sel_1}
  Given an MDP $M$ and a set of representations $\{\Phi_j\}_{j\in[N]}$ satisfying the low-rank assumption (Asm.~\ref{asm:lowrank}). For all $k \in \mathbb{N},h \in [H]$, with probability $1-2\delta$:
\begin{align}\label{equation::multi_representation_lemmaB3}
    \norm{\sum_{i=1}^k\phi^{(j)}_h(s_h^i, a_h^i)\left(\overline{V}^k_{h+1}(s_{h+1}^i)-\Prob_h \overline{V}^k_{h+1}(s_h^i,a_h^i)\right)}_{\Lambda_{h,k}^{-1}(j)} \le C d H \sqrt{N\log(2N(c_\beta + 1)dHk/\delta)},
\end{align}
for all $j \in [N]$ and for some constant $C$ independent of $c_\beta$.

\end{lemma}

\begin{proof}
This result follows from a simple use of an anytime version of Lemma D.4 from~\citep{Jin2020linear} with $\epsilon=dH/k$ and $\delta'=\frac{\delta}{2N}$ and $\lambda=1$. Let $j \in [N]$ be one of the representations. 
\begin{align*}
     &\norm{\sum_{i=1}^k\phi_h^{(j)}(s_h^i, a_h^i)\left(\overline{V}^k_{h+1}(s_{h+1}^i)-\Prob_h\overline{V}^k_{h+1}(s_h^i,a_h^i)\right)}_{\Lambda_{h,k}^{-1}(j)}^2\\ &\le 4H^2\Bigg[\frac{d}{2}\log\left(\frac{k+\lambda}{\lambda}\right)+2\log\frac{\pi k}{\sqrt{6}} + \log\frac{2}{\delta} + dN\log\left(1+\frac{8k^{3/2}}{\sqrt{\lambda d}}\right) + \nonumber\\&\quad d^2 N\log\left(1+\frac{8\sqrt{d}c_\beta^2k^2\log(2dHk/\delta)}{\lambda}\right)\Bigg] + \frac{8d^2H^2}{\lambda} \\&
     = \mathcal{O}(d^2 N H^2\log(2N(c_\beta+1)dHk/\delta))
\end{align*}
A simple union bound over all representations in $\{\Phi_j\}_{j\in[N]}$ yields the desired result. 

\end{proof}

We have now the necessary ingredients to prove an equivalent version to Lemma B.4 from~\citep{Jin2020linear} for the case of multiple representations.

\begin{lemma}[Equivalent to Lemma B.4 in~\citep{Jin2020linear}]\label{lemma::representation_sel_2} Given an MDP $M$ and a set of representations $\{\Phi_j\}_{j\in[N]}$ satisfying the low-rank assumption (Asm.~\ref{asm:lowrank}). With probability at least $1-2\delta$, for any policy $\pi$, any episode $k\in\mathbb{N}$, stage $h\in[H]$, state $s\in\mathcal{S}$ and action $a \in \mathcal{A}$,
    \begin{align*}
    \left|  \langle \phi_h^{(j)}(s,a), \boldsymbol{w}_h^k(j) \rangle - Q_h^{\pi}(s,a) -   \mathbb{P}_h\left(  \overline{V}_{h+1}^k - V_{h+1}^\pi   \right)(s,a)         \right| &\leq \beta_k  \norm{ \phi^{(j)}(s,a)   }_{\Lambda_{h, k}(j)^{-1}  }
    \end{align*}
where $\beta_k = C' d H \sqrt{N\log(2N(c_\beta + 1)dHk/\delta)} $. For some absolute constant $C'$. 
\end{lemma}
\begin{proof}
We know that for any $(s,a,h) \in \mathcal{S} \times \mathcal{A} \times [H]$:
\begin{equation*}
    Q_h^\pi( s,a) = \langle \phi_h^{(j)}(s,a), \boldsymbol{w}_h^\pi(j) \rangle = \left( r_h + \mathbb{P}_h V_{h+1}^\pi \right)(s,a)  \quad \forall j \in [N],
\end{equation*}
This gives 

\begin{align*}
    \boldsymbol{w}_h^k(j) - \boldsymbol{w}_h^\pi(j) &=   \Lambda^{-1}_{h,k}(j) \sum_{i=1}^{k-1}  \phi_h^{(j)}(s_h^i, a_h^i)\left(   r_h(s_h^i,a_h^i)+\max_{a\in\Aspace}\overline{Q}^{k-1}_{h+1}(s_{h+1}^i,a)\right) - \boldsymbol{w}_h^{\pi} \\
    &=   \Lambda_{h,k}(j)^{-1} \left\{  -\lambda \boldsymbol{w}_h^\pi + \sum_{i=1}^{k-1} \phi_h^{(j)}(s_h^i, a_h^i) \left(\overline{V}_{h+1}^k(s_{h+1}^i) - \mathbb{P}_h V_{h+1}^\pi(s_h^i, a_h^i) \right)     \right\} \\
    &=
    \underbrace{- \lambda \Lambda^{-1}_{h, k}(j)  \boldsymbol{w}_h^{\pi}(j) }_{\boldsymbol{q}_1}  +  \underbrace{ \Lambda_{h, k}^{-1}(j) \sum_{i=1}^{k-1} \phi_h^{(j)}(s_h^i, a_h^i) \left(       \overline{V}_{h+1}^k(s_{h+1}^i) - \mathbb{P}_h \overline{V}_{h+1}^k(s_{h}^i, a_h^i) \right)    }_{\boldsymbol{q}_2} +\\
    &\quad \underbrace{\Lambda^{-1}_{h,k}(j) \left( \sum_{i=1}^{k-1} \phi_h^{(j)}(s_h^i, a_h^i) \mathbb{P}_h\left(  \overline{V}_{h+1}^k - V_{h+1}^\pi   \right)(s_h^i, a_h^i) \right) }_{\boldsymbol{q}_3}  
\end{align*}
Now we bound the terms on the right hand side. For the first term,
\begin{equation*}
    \left|\langle \phi_h^{(j)}(s,a), \boldsymbol{q}_1\rangle   \right| = \left| \lambda \langle \phi_h^{(j)}(s,a), \Lambda_{h, k}^{-1}(j)  \boldsymbol{w}_h^\pi \rangle \right| \leq \sqrt{\lambda} \| \boldsymbol{w}_h^\pi \|  \norm{ \phi_h^{(j)}(s,a)   }_{\Lambda_{h, k}^{-1}(j)  } \stackrel{(i)}{\leq} 2H\sqrt{ d \lambda} \norm{ \phi_h^{(j)}(s,a)   }_{\Lambda_{h, k}^{-1}(j)  }
\end{equation*}
Inequality $(i)$ above holds because of Lemma B.1 of~\citep{Jin2020linear}. For the second term $\boldsymbol{q}_2$, given the event defined in Lemma~\ref{lemma::representation_sel_1} (which holds with probability at least $1-2\delta$) we have,
\begin{equation*}
    \left|   \langle \phi_h^{(j)}(s,a), \boldsymbol{q}_2 \rangle \right| \leq C d H \sqrt{N\log(2N(c_\beta + 1)dHk/\delta)} \norm{ \phi_h^{(j)}(s,a)   }_{\Lambda_{h, k}^{-1}(j)  }
\end{equation*}
For the third term,
\begin{align*}
    &\langle \phi_h^{(j)}(s,a
    ), \boldsymbol{q}_3 \rangle \\
    &= \left\langle \phi_h^{(j)}(s,a), \left(\Lambda^{-1}_{h,k}(j) \right)\sum_{i=1}^{k-1} \phi_h^{(j)}(s_h^i, a_h^i)\mathbb{P}_h\left( \overline{V}_{h+1}^k - V_{h+1}^\pi \right)(s_h^i, a_h^i) \right\rangle\\
    &=\left\langle \phi_h^{(j)}(s,a), \left(\Lambda^{-1}_{h,k}(j) \right)\sum_{i=1}^{k-1} \phi_h^{(j)}(s_h^i, a_h^i)\phi^\top_j(s_h^i, a_h^i)\int \left( \overline{V}_{h+1}^k - V_{h+1}^\pi \right)(s'_{h+1}) d\boldsymbol{\mu}^j_h(s'_{h+1} | s_h^i, a_h^i) \right\rangle \\
    &= \underbrace{\left\langle \phi_h^{(j)}(s,a), \int \left( \overline{V}_{h+1}^k - V_{h+1}^\pi \right)(s'_{h+1}) d\boldsymbol{\mu}^j_h(s'_{h+1} | s_h^i, a_h^i) \right\rangle}_{p_1} - \\
    &\quad \underbrace{\lambda \left \langle \phi_h^{(j)}(s,a), \Lambda^{-1}_{h,k}(j) \int \left( \overline{V}_{h+1}^k - V_{h+1}^\pi \right)(s'_{h+1}) d\boldsymbol{\mu}^j_h(s'_{h+1} | s_h^i, a_h^i)   \right \rangle}_{p_2}
\end{align*}

And therefore,

\begin{equation*}
    p_1 = \mathbb{P}_h\left(  \overline{V}_{h+1}^k - V_{h+1}^\pi   \right)(s,a), \qquad | p_2 | \leq 2H \sqrt{d\lambda} \norm{ \phi_h^{(j)}(s,a)   }_{\Lambda_{h, k}^{-1}(j)  }
\end{equation*}

Finally since $\langle \phi_h^{(j)}(s,a), \boldsymbol{w}_h^k(j) \rangle - Q_h^{\pi}(s,a) = \langle \phi_h^{(j)}(s,a), \boldsymbol{w}_h^k - \boldsymbol{w}_h^{\pi} \rangle = \langle \phi_h^{(j)}(s,a) , \boldsymbol{q}_1 + \boldsymbol{q}_2 + \boldsymbol{q}_3 \rangle$, we have
\begin{align*}
    &\left|  \langle \phi_h^{(j)}(s,a), \boldsymbol{w}_h^k(j) \rangle - Q_h^{\pi}(s,a) -   \mathbb{P}_h\left(  \overline{V}_{h+1}^k - V_{h+1}^\pi   \right)(s,a)         \right|\\
    &\leq \left( C d H \sqrt{N\log(2N(c_\beta + 1)dHk/\delta)} + 4H\sqrt{d\lambda} \right) \norm{ \phi_h^{(j)}(s,a)   }_{\Lambda_{h, k}^{-1}(j)  } \\
    &\leq  C' d H \sqrt{N\log(2N(c_\beta + 1)dHk/\delta)}  \norm{ \phi_h^{(j)}(s,a)   }_{\Lambda_{h, k}^{-1}(j)  }
\end{align*}

For some constant $C'$. The result follows.

\end{proof}

\begin{lemma}\label{lemma::peeling_lemma_mixing} Given an MDP $M$ and a set of representations $\{\Phi_j\}_{j\in[N]}$ satisfying the low-rank assumption (Asm.~\ref{asm:lowrank}). With probability at least $1-2\delta$, for any episode $k\in \mathbb{N}$, stage $h\in[H]$, and state $s\in\mathcal{S}$,
    \begin{align*}
    \overline{V}_h^k(s) - V_h^{\pi^k}(s) \leq 2\beta_{k} \min_{j\in[N]} \|\phi_h^{(j)}(s,\pi_h^k(s)) \|_{\Lambda^{-1}_{h,k}(j)} + \mathbb{E}_{s'\sim p_h(s,\pi_h^k(s))} \Big[ \overline{V}_{h+1}^k(s') - V_{h+1}^{\pi^k}(s') \Big].
    \end{align*}
    Where $\beta_{k} =  C' d H \sqrt{N\log(2N(c_\beta + 1)dHk/\delta)}$.  
\end{lemma}
\begin{proof}
Note that $\overline{V}_h^k(s) - V_h^{\pi^k}(s) = \overline{Q}_h^k(s,\pi_h^k(s)) - Q_h^{\pi^k}(s,\pi_h^k(s))$. Using Lemma~\ref{lemma::representation_sel_2}, for any $j \in [N]$
\begin{align*}
Q_h^{\pi^k}(s,\pi_h^k(s)) &\geq \langle \phi_h^{(j)}(s,\pi_h^k(s)), \boldsymbol{w}_h^k(j)\rangle - \\
&\quad\mathbb{E}_{s'\sim p_h(s,\pi_h^k(s))}[\overline{V}_{h+1}^k(s') - V_{h+1}^{\pi^k}(s')] - \beta_{h,k}\|\phi_h^{(j)}(s,\pi_h^k(s))\|_{\Lambda^{-1}_{h,k}(j)} \\ 
\end{align*}
And therefore for all $j \in [N]$,
\begin{align*}
    \langle \phi_h^{(j)}(s,\pi_h^k(s)), \boldsymbol{w}_h^k(j)\rangle +     \beta_{h,k}\|\phi_h^{(j)}(s,\pi_h^k(s))\|_{\Lambda^{-1}_{h,k}(j)}  -V_h^{\pi^k}(s) \leq\\ 2\beta_{h,k}\|\phi_h^{(j)}(s,\pi_h^k(s))\|_{\Lambda^{-1}_{h,k}(j)} + \mathbb{E}_{s'\sim p_h(s,\pi_h^k(s))}[\overline{V}_{h+1}^k(s') - V_{h+1}^{\pi^k}(s')]
\end{align*}

Taking the minimum over $j \in [N]$ (and $H$) on the LHS yields the result,
\begin{align*}
\overline{V}_h^k(s) - V_h^{\pi^k}(s) \leq 2\beta_{h,k} \min_{j\in[N]} \|\phi_h^{(j)}(s,\pi_h^k(s)) \|_{\Lambda^{-1}_{h,k}(j)}  + \mathbb{E}_{s'\sim p_h(s,\pi_h^k(s))}[\overline{V}_{h+1}^k(s') - V_{h+1}^{\pi^k}(s')].
\end{align*}

\end{proof}

Finally we show this implies optimism holds,

\begin{lemma}\label{lemma::optimism_leader}[Optimism. Equivalent version of Lemma B.5 in~\citep{Jin2020linear}]
    With probability $1-\delta$ and for all $s,a \in \mathcal{S} \times \mathcal{A}$, $k \in \mathbb{N}$ and $h \in [H]$, the $\{\overline{Q}_h^{k}\}_{h \in [H]}$ functions of LSVI-LEADER satisfy,
    \begin{equation*}
        \overline{Q}_{h}^k(s,a) \geq Q_h^{*}(s,a)
.     \end{equation*}
\end{lemma}

\begin{proof}

    The same proof as in Lemma B.5 in~\citep{Jin2020linear} works just simply modifying it to have a minimum over $j \in [N]$ in the necessary places. We reproduce the argument here for completeness. The proof of the Lemma proceeds by induction.
    
    First, we prove the base case, at the last step $H$. The statement holds because $\overline{Q}_H^{k}(s,a) \geq Q_H^*(s,a)$ since the value function at $H+1$ is zero and by Lemma~\ref{lemma::representation_sel_2} we have that with probability at least $1-2\delta$ for all $k \in \mathbb{N}$, $s \in \mathcal{S}, a \in \mathcal{A}$ and any $j \in [N]$,
    
    \begin{equation*}
         \left|  \langle \phi_h^{(j)}(s,a), \boldsymbol{w}_{H}^k(j) \rangle - Q_H^{\pi_*}(s,a)    \right| \leq  C' d H \sqrt{N\log(2N(c_\beta + 1)dHk/\delta)}  \norm{ \phi^{(j)}_h(s,a)   }_{\Lambda_{H, k}^{-1}(j)  }
    \end{equation*}
    
    Therefore for all $j \in [N]$, with probability at least $1-2\delta$,
    \begin{equation*}
        \langle \phi_h^{(j)}(s,a), \boldsymbol{w}_{H}^k(j) \rangle + C' d H \sqrt{N\log(2N(c_\beta + 1)dHk/\delta)}  \norm{ \phi^{(j)}_h(s,a)   }_{\Lambda_{H, k}^{-1}(j)  } \geq  Q_H^{\pi_*}(s,a)   
    \end{equation*}
    Since $H \geq Q_H^{\pi_*}(s,a)$ by definition, we conclude that taking the mimimum over $j \in [N]$ (and $H$), and using the fact that \begin{small}$$\overline{Q}_h^{k}(s,a) = \min\left(\min_{j\in[N]}    \langle \phi_h^{(j)}(s,a), \boldsymbol{w}_{H}^k(j) \rangle + C' d H \sqrt{N\log(2N(c_\beta + 1)dHk/\delta)}  \norm{ \phi_h^{(j)}(s,a)   }_{\Lambda_{H, k}^{-1}(j)  }     , H\right)$$\end{small}
    We conclude that,
    \begin{equation*}
        \overline{Q}_H^{k}(s,a) \geq Q_H^{\pi_*}(s,a).
    \end{equation*}

    Now, suppose the statement holds true at step $h+1$ and consider step $h$. Again by Lemma~\ref{lemma::representation_sel_2} we have, for all $k \in [K]$ and all $j \in [N]$
    \begin{align*}
         & \left|  \langle \phi_h^{(j)}(s,a), \boldsymbol{w}_h^k(j) \rangle - Q_h^{\pi_*}(s,a) -   \mathbb{P}_h\left(  \overline{V}_{h+1}^k - V_{h+1}^{\pi_*}   \right)(s,a)         \right| \\
         &\leq  C' d H \sqrt{N\log(2N(c_\beta + 1)dHk/\delta)}  \norm{ \phi_j(s,a)   }_{\Lambda_{h, k}^{-1}(j)  }
    \end{align*}
    By the induction assumption that $\mathbb{P}_h\left(  \overline{V}_{h+1}^k -V_{h+1}^{\pi_*}   \right)(s,a) \geq 0$, we have for all $j \in [N]$:
    \begin{small}
    \begin{equation*}
        Q_h^{\pi_*}(s,a) \leq \min\left(  \langle \phi_h^{(j)}(s,a), \boldsymbol{w}_h^k(j) \rangle +    C' d H \sqrt{N\log(2N(c_\beta + 1)dHk/\delta)}  \norm{ \phi_h^{(j)}(s,a)   }_{\Lambda_{h, k}^{-1}(j)  }       , H \right)
    \end{equation*}
    \end{small}
    The result follows by taking a minimum over $j \in [N]$.
\end{proof}




\paragraph{Finishing the proof of Theorem~\ref{thm:mix.stage}.} Having proven Lemma~\ref{lemma::peeling_lemma_mixing} and that optimism holds for LSVI-LEADER (Lemma~\ref{lemma::optimism_leader}), we conclude that an equivalent version of Assumption~\ref{asm:algo.strong} holds. The same logic of the proofs of Lemmas~\ref{lem:general.worstcase},~\ref{lem:gapbound} and Theorem~\ref{thm:general.const} apply in this case. Hence, we conclude that the regret of LSVI-LEADER is upper bounded by the minimum of these regret bounds for all representations $z\in\mathcal{Z}$, thus proving the first result. To obtain the second result, simply notice that, if $z^\star \in \mathcal{Z}$ is \hls{}, then we can use the refined analysis for LSVI-UCB of Thm.~\ref{thm:const.lsvi} to show that $\wt{R}(K,z^\star,\{\beta_k\})$ is upper bounded by a constant independent of $K$, hence proving constant regret for LSVI-LEADER.

\paragraph{Proof of Theorem~\ref{thm:mix.state}.}  The proof follows the template of Thm~\ref{thm:const.lsvi}, but as shown in Lemma~\ref{lemma::peeling_lemma_mixing}, the confidence sets of LSVI-LEADER scale with the minimum w.r.t. $j$ of the feature norms. In place of Equation~\ref{eq:shrink}, and with the aid of Lemma~\ref{lem:key} we see that since the collection of feature maps $\{ \Phi_j\}_{j\in[M]}$ is \hls{}-mixing for all reachable $s,a$:
\begin{align}
    \beta_k \min_{j \in [N]}\norm{\phi^{(j)}_h(s,a)}_{\Lambda^{-1}_{h,k}(j)} \le   &\beta_k \frac{k + \lambda-g(k)-8\sqrt{k\log(2NdHk/\delta)}}{(k\lambda^+(h,s,a)+\lambda-g(k) -8\sqrt{k\log(2NdHk/\delta)})^{3/2}} \label{eq:mix.state.main}\\
    &= \wt{O}(k^{-1/2}),\nonumber
\end{align}
where $g(k)=\wt{O}(\sqrt{k})$ is the regret upper bound from Thm.~\ref{thm:mix.stage},
\begin{equation}
    \lambda^+(h,s,a)=\max_{j\in\mathcal{J}(h,s,a)}  \lambda^+_{h,j},
\end{equation}
and $\mathcal{J}(h,s,a) \subseteq [N]$ is such that $j\in \mathcal{J}(h,s,a)$ if $\phi^{(j)}_h(s,a) \in \mathrm{span}\left\{\phi_h^{(j)}(s, \pi^*_h(s)) | \rho^{\star}_h(s) > 0\right\}$. 
To see this, notice that we can instantiate Lemma~\ref{lem:key} with any representation $j\in[N]$ such that $\phi^{(j)}_h(s,a)$ belongs to the span of optimal features. So we use the representation with the largest eigenvalue $\lambda_{h,j}^{+}$. The \hls{}-mixing property (Def.~\ref{def:mixing-hls}) guarantees $\mathcal{J}(h,s,a)$ is always nonempty.

By~\eqref{eq:mix.state.main} and Lemma~\ref{lem:gapbound} (where $C_h=0$ thanks to local optimism), for each $h\in[H]$ there exists an episode $\kappa_h$ independent of $K$ such that, for all reachable $s$ and $k>\kappa_h$: 
\begin{align}
    \Delta_h(s,\pi^{k}_h(s)) 
    &\le 2\beta_{k}\EV_{\pi^{k}}\left[\sum_{i=h}^H \frac{k+ \lambda-g(k)-8\sqrt{k\log(2NdHk/\delta)}}{(k\lambda^+(i,s_{i},a_{i}) + \lambda-g(k)-8\sqrt{k\log(2NdHk/\delta)})^{3/2}}\bigg|s_h=s\right] \nonumber\\
    &< \Delta_{\min}.
\end{align}
So after $\wt{\kappa}=\max_{h} \{\kappa_h\}$ episodes, LSVI-UCB suffers zero regret. Finally, the regret up to $\wt{\kappa}$ cannot be worse than that obtained in Thm.~\ref{thm:mix.stage} without the \hls{}-mixing property.

\section{Auxiliary Results}\label{app:auxiliary}

\begin{proposition}[Azuma's inequality]\label{prop:azuma}
Let $\{(Z_t,\mathcal{F}_t)\}_{t\in\mathbb{N}}$ be a martingale difference sequence such that $|Z_t| \leq a$ almost surely for all $t\in\mathbb{N}$. Then, for all $\delta \in (0,1)$,
\begin{align}
\mathbb{P}\left(\forall t \geq 1 : \left|\sum_{k=1}^t Z_k \right| \leq a\sqrt{t \log(2t/\delta)} \right) \geq 1-\delta.
\end{align}
\end{proposition}

\begin{proposition}[Matrix Azuma,~\citealp{tropp2012user}]\label{prop:mazuma}
	Let $\{X_k\}_{k=1}^t$ be a finite adapted sequence of symmetric matrices of dimension $d$, and $\{C_k\}_{k=1}^t$ a sequence of symmetric matrices such that for all $k$, $\EV_{k}[X_k]=0$ and $X_k^2\preceq C_k^2$ almost surely. Then, with probability at least $1-\delta$:
	\begin{equation}
	\lambda_{\max}\left(\sum_{k=1}^tX_k\right) \le \sqrt{8\sigma^2\log(d/\delta)},
	\end{equation}
	where $\sigma^2=\norm{\sum_{k=1}^tC_k^2}$.
\end{proposition}

\begin{proposition}[\cite{He2020loglin}]\label{prop:sumofgaps}
For any $h\in[H]$, $s\in\mathcal{S}$, and $\pi\in\Pi$:
\begin{equation*}
    V^\star_h(s) - V^{\pi}_h(s) = \EV_{\pi}\left[\sum_{i=h}^H\Delta_i(s_i,a_i) \Bigg| s_h=s\right],
\end{equation*}
Hence the regret after $K$ episodes can be expressed as:
\begin{equation*}
    R(K) = \sum_{k=1}^K V_1^\star(s_1^k) - V_1^{\pi^k}(s_1^k) = \sum_{k=1}^K\EV_{\pi^k}\left[\sum_{h=1}^H\Delta_h(s_h,a_h)\Bigg|s_1=s_1^k\right].
\end{equation*}
\end{proposition}
\begin{proof}
    By definition of $\Delta_h$:
    \begin{align}
        V^\star_h(s) - V^{\pi}_h(s) &= Q^\star_h(s,\pi_h(s)) + \Delta_h(s,\pi_h(s)) - V^{\pi}_h(s) \\
        &= r_h(s,\pi_h(s)) + \EV_{s'\sim p_h(s,\pi_h(s))}[V^\star_{h+1}(s')] + \Delta_h(s,\pi_h(s)) - r_h(s,\pi_h(s)) \nonumber\\&\qquad- \EV_{s'\sim\Prob_h(s_h,\pi_h(s_h))}[V^\pi_{h+1}(s')] \\
        &= \Delta_h(s_h,\pi_h(s_h)) + \EV_{s'\sim\Prob_h(s_h,\pi_h(s_h))}[V^\star_{h+1}(s') - V^\pi_{h+1}(s')].
    \end{align}
    Unrolling the recursion up to $H$ concludes the proof.
\end{proof}

\begin{lemma}\label{lem:loewner}
Assume $R(k)\le g(k)$ for all $k\ge 1$ and Asm.~\ref{asm:gap} holds. Then, probability $1-\delta$, for all $h,k$:
    \begin{equation}
        \Lambda_h^{k+1} \succeq k\Lambda_h^\star + \lambda I - \Delta_{\min}^{-1}g(k) I - 8L^2I\sqrt{k\log(2dkH/\delta)}.
    \end{equation}
\end{lemma}
\begin{proof}
Define a trajectory as a sequence of states and actions $\tau_h=(s_1,a_1,\dots,s_h,a_h)$. Let $\Gamma_h$ denote the set of all trajectories of length $h$. The distribution over trajectories induced by a (deterministic) policy $\pi$ is $p^{\pi}_h(\tau_h)=\mu(s_1)\indi{a_1=\pi_1(s_1)}p_1(s_2|s_1,a_1)\dots p_{h-1}(s_h|s_{h-1},a_{h-1})\indi{a_h=\pi_h(s_h)}$. We abbreviate as $p_h^\star$ the distribution induced by the optimal policy $\pi^\star$ and as $p_h^k$ the one induced by $\pi^k$, the algorithm's policy at episode $k$. 
Let us define the following event:
\begin{equation}
    E_h^k = \{\tau\in\Gamma_h \text{ s.t. } a_i=\pi_h^k(s_i)=\pi^\star_h(s_i) \text{ for } i=1,\dots, h\}.
\end{equation}
Then:
\begin{align}
    \Lambda_h^{k+1} - \lambda I 
    &= \sum_{i=1}^k \phi(s_h^i,a_h^i)\phi(s_h^i,a_h^i)^\transp \nonumber\\
    &\succeq \sum_{i=1}^k 
    \indi{\tau_h^i\in E_h^i}
    \phi(s_h^i,a_h^i)\phi(s_h^i,a_h^i)^\transp \nonumber\\
    &=\sum_{i=1}^k \indi{\tau_h^i\in E_h^i}\phi_h^\star(s_h^i)\phi_h^\star(s_h^i)^\transp \label{pp:loewner.1} \\
    &=\underbrace{\sum_{i=1}^k \EV_{\tau_h\sim p_h^i}\left[\indi{\tau_h\in E_h^i}\phi_h^\star(s_h)\phi_h^\star(s_h)^\transp\right]}_{(A)} \nonumber\\
    &\qquad+ \underbrace{\sum_{i=1}^k \left(
    \indi{\tau_h^i\in E_h^i}\phi_h^\star(s_h^i)\phi_h^\star(s_h^i)^\transp
    -\EV_{\tau_h\sim p_h^i}\left[\indi{\tau_h\in E_h^i}\phi_h^\star(s_h)\phi_h^\star(s_h)^\transp\right] 
    \right)}_{\text{(B)}}, \nonumber
    \end{align}
    where~\eqref{pp:loewner.1} is by definition of $E_h^i$ and expectations are conditioned on history up to the beginning of the $i$-th episode. We first bound $(B)$ with a matrix version of Azuma's inequality. Let:
    \begin{equation*}
        X_h^i = \indi{\tau_h^i\in E_h^i}\phi_h^\star(s_h^i)\phi_h^\star(s_h^i)^\transp
    -\EV_{\tau_h\sim p_h^i}\left[\indi{\tau_h\in E_h^i}\phi_h^\star(s_h)\phi_h^\star(s_h)^\transp\right].
    \end{equation*}
    Clearly $\EV[X_h^i]=0$. Moreover, since $X_h^i$ is symmetric:
    \begin{equation}
        (X_h^i)^2 \preceq \lambda_{\max}((X_h^i)^2)I \preceq \norm{X_h^i}^2 I \preceq 4I.
    \end{equation}
    Then by Proposition~\ref{prop:mazuma}, with probability $1-\delta_h^k$:
    \begin{equation}
        \lambda_{\max}\left(\sum_{i=1}^k X_h^i\right) \le 4\sqrt{2k\log(d/\delta_h^k)}.
    \end{equation}
    Setting $\delta_h^k=\delta/(2Hk^2)$ we can perform a union bound over episodes and stages to obtain, with probability $1-\delta$, for all $h,k$:
    \begin{equation}
        (B) = \sum_{i=1}^kX_h^i \preceq \lambda_{\max}\left(\sum_{i=1}^k X_h^i\right)I \preceq 8I\sqrt{k\log(2dHk/\delta)}.
    \end{equation}
    Now we focus on the $(A)$ term. First, observe that the probability measures $p_h^k$ and $p_h^\star$ agree on $E_h^k$. Indeed, if $\tau_h\in E_h^k$:
    \begin{align}
        p_h^k(\tau_h) &= \mu(s_1)\indi{a_1=\pi^k_1(s_1)}p_1(s_2|s_1,a_1)\dots p_{h-1}(s_h|s_{h-1},a_{h-1})\indi{a_h=\pi^k_h(s_h)}\nonumber\\
        &=\mu(s_1)\indi{a_1=\pi^\star_1(s_1)}p_1(s_2|s_1,a_1)\dots p_{h-1}(s_h|s_{h-1},a_{h-1})\indi{a_h=\pi^\star_h(s_h)}\\
        &=\mu(s_1)p_1(s_2|s_1,a_1)\dots p_{h-1}(s_h|s_{h-1},a_{h-1}).
    \end{align}
    So:
    \begin{align}
    (A)&=\sum_{i=1}^{k}\EV_{\tau_h\sim p_h^i}[\indi{\tau_h\in E_h^i}\phi^\star_h(s_h)\phi^\star_h(s_h)^\transp] \nonumber\\
    &= \sum_{i=1}^{k}\EV_{\tau_h\sim p_{h}^\star}[\indi{\tau_h\in E_h^i}\phi^\star_h(s_h)\phi^\star_h(s_h)^\transp] \\
    &= k\EV_{\tau_h\sim p_{h}^\star}[\phi^\star_h(s_h)\phi^\star_h(s_h)^\transp] - \sum_{i=1}^k\int_{\Gamma_h\setminus E_h^i}\phi^\star_h(s_h)\phi^\star_h(s_h)^\transp p^\star_{h}(\de\tau_h)\\
    &= k\EV_{s\sim \rho_{h}^\star}[\phi^\star_h(s_h)\phi^\star_h(s_h)^\transp] - \sum_{i=1}^k\int_{\Gamma_h\setminus E_h^i}\phi^\star_h(s_h)\phi^\star_h(s_h)^\transp p^\star_{h}(\de\tau_h)\\
    &\succeq k\EV_{s\sim \rho_{h}^\star}[\phi^\star_h(s_h)\phi^\star_h(s_h)^\transp] - I\sum_{i=1}^k\left(1-\int_{E_h^i}p^\star_{h}(\de\tau_h)\right)
    \\
    &= k\EV_{s\sim \rho_{h}^\star}[\phi^\star_h(s_h)\phi^\star_h(s_h)^\transp] - I\sum_{i=1}^k\left(1-\int_{E_h^i}p^\star_{h}(\de\tau_h)\right)
    \\
    &= k\EV_{s\sim \rho_{h}^\star}[\phi^\star_h(s_h)\phi^\star_h(s_h)^\transp] - I\underbrace{\sum_{i=1}^k\EV_{\tau_h\sim p_h^i(\tau_h)}[\indi{\tau_h\notin E_h^i}]}_{(C)}.
\end{align}
Finally, under Asm.~\ref{asm:gap} and the regret upper bound:
\begin{align}
    (C) &= \sum_{i=1}^k\EV_{\tau_h\sim p_h^i(\tau_h)}[\indi{\tau_h\notin E_h^i}] \nonumber\\
    &\le \sum_{i=1}^k\sum_{j=1}^{h} \EV_{\pi^i}[\indi{a_j\neq \pi^\star_j(s_j)}] \label{pp:loewner.2}\\
    &\le \sum_{i=1}^k\sum_{j=1}^{h} \EV_{\pi^i}[\indi{\Delta_j(s_j,a_j)\ge \Delta}] \label{pp:loewner.3}\\
    &\le \sum_{i=1}^k\sum_{j=1}^{h} \EV_{\pi^i}\left[\frac{\Delta_j(s_j,a_j)}{\Delta_{\min}}\right] \\
    &= \frac{1}{\Delta_{\min}}\sum_{i=1}^k\EV_{\pi^i}\left[\sum_{j=1}^{h} \Delta_j(s_j,a_j)\right] \\
    &\le \frac{1}{\Delta_{\min}}\sum_{i=1}^k\EV_{\pi^i}\left[\sum_{h=1}^{H} \Delta_h(s_h,a_h)\right] \\
    &= \frac{R(k)}{\Delta_{\min}} \le \frac{g(k)}{\Delta_{\min}}, \label{pp:loewner.4}
\end{align}
where~\eqref{pp:loewner.2} is by definition of $E_h^i$, \eqref{pp:loewner.3} is from the uniqueness of the optimal policy and Asm.~\ref{asm:gap}, and~\eqref{pp:loewner.4} is from Proposition~\ref{prop:sumofgaps}.
\end{proof}

\begin{proposition}[Lemma 29 from~\citep{papini2021leveraging}]\label{prop:kanto}
	Let $\boldsymbol{v}\in\Reals^d$ with $\norm{\boldsymbol{v}}=1$ and $A\in\Reals^{d\times d}$ symmetric invertible with non-zero eigenvalues $\lambda_1\le\dots\le\lambda_d$ and corresponding orthonormal eigenvectors $u_1,\dots,u_d$. Let $\mathcal{I}\subseteq[d]$ be any index set. If $\boldsymbol{v}\in\spann\{u_i\}_{i\in \mathcal{I}}$ and $\lambda_i>0$ for all $i\in\mathcal{I}$:
	\begin{equation*}
	\boldsymbol{v}^\transp A^{-1}\boldsymbol{v} \le \frac{(\max_{i\in \mathcal{I}}\lambda_i+\min_{i\in \mathcal{I}}\lambda_i)^2}{4\max_{i\in \mathcal{I}}\lambda_i\min_{i\in \mathcal{I}}\lambda_i} \frac{1}{\boldsymbol{v}^\transp A\boldsymbol{v}}.
	\end{equation*}
\end{proposition}

\begin{proposition}[e.g., Lemma 30 from~\citep{papini2021leveraging}]\label{prop:minposeig}
	The smallest nonzero eigenvalue of a symmetric p.s.d. matrix $A\in\Reals^{d\times d}$ is:
	\begin{equation*}
	\lambda_{\min}^+(A) = \min_{\substack{\boldsymbol{v}\in\Imm(A)\\\norm{\boldsymbol{v}}=1}}\boldsymbol{v}^\transp A\boldsymbol{v},
	\end{equation*}
	where $\Imm(A)$ denotes the column space of $A$.
\end{proposition}

\begin{lemma}\label{lem:key}
Consider a $d$-dimensional representation $(\phi_h)_{h\in[H]}$.
Assume there exists an increasing $\wt{O}(\sqrt{k})$ function $g$ such that $R(k)\le g(k)$ for all $k\ge 1$, Asm. \ref{asm:gap} holds, and $\beta_k=\wt{O}(1)$. Then with probability $1-\delta$, for all $h$, there exists a constant $\wt{\kappa}_h$ such that, for every $k\geq \wt{\kappa}_h$ and all $s,a$ having  $\phi_h(s,a) \in \mathrm{span}\left\{\phi_h^\star(s) | \rho^{\star}_h(s) > 0\right\}$,
\begin{equation*}
    \beta_k\norm{\phi_h(s,a)}_{(\Lambda_h^k)^{-1}} \le \beta_k\frac{k+\lambda-g(k)-8\sqrt{k\log(2dHk/\delta)}}{(k\lambda_{h}^{+}+\lambda - g(k)-8\sqrt{k\log(2dHk/\delta)})^{3/2}} = \wt{\mathcal{O}}(k^{-1/2}),
\end{equation*}
where $\lambda_h^+$ is the minimum nonzero eigenvalue of $\Lambda_h^\star$.
\end{lemma}

\begin{proof}
    We follow the proof scheme of Lemma 19 from~\citep{papini2021leveraging}.
    Let $f(k)=g(k) + 8\sqrt{k\log(2dHk/\delta)} = \wt{O}(\sqrt{k})$. Notice that $f(k)$ is positive.
    
    Fix $h$ and let $B_h^k = k\Lambda_h^\star + \lambda I - f(k)I$. First, notice that $B_h^k$ is an affine transformation of $\Lambda_h^\star$. As such, $B_h^k$ has the same orthonormal eigenvectors as $\Lambda_h^\star$, and we can define a mapping between the eigenvalues of the two matrices. Next, notice that $B_h^k$ is always invertible for sufficiently large $k$. Indeed, zero eigenvalues of $\Lambda_h^\star$ are mapped to negative eigenvalues of $B_h^k$ for sufficiently large $k$ --- and since $f(k)$ is increasing and sublinear, positive eigenvalues of $\Lambda_h^\star$ are mapped to positive eigenvalues of $B_h^k$ for sufficiently large $k$.
    We call $\wt{\kappa}_h$ the smallest $k$ such as both conditions hold.
    For the rest of the proof assume $k\ge \kappa_h$.
    We have shown that $B_h^k$ is invertible and all and only the nonzero eigenvalues of $\Lambda_h^\star$ are mapped into positive eigenvalues of $B_h^k$, with the same orthonormal eigenvectors.
    
    Now fix $(s,a)$ such that $\phi_h(s,a) \in \mathrm{span}\left\{\phi_h^\star(s) | \rho^{\star}_h(s) > 0\right\}$ and let $x = \phi_h(s,a)/\norm{\phi_h(s,a)}$. From Lemma~\ref{lem:loewner}, with probability $1-\delta$, $\Lambda_h^k \succeq B_h^k$. So:
    \begin{equation}
        x^\transp (\Lambda_h^k)^{-1}x \leq
        x^\transp (B_h^k)^{-1}x.
    \end{equation}
    By hypothesis $x$ belongs to the column space $\Imm(\Lambda_h^\star)$, so it belongs to the span of $\wt{d}\le d$ orthonormal eigenvectors of $\Lambda_h^\star$. From the properties of $B_h^k$ stated above, $x$ belongs to the span of $\wt{d}$ orthonormal eigenvectors of $B_h^k$ corresponding to positive eigenvalues.
    The smallest such eigenvalue is:
    \begin{equation}
        k\lambda_{h}^{+} + \lambda - f(k),
    \end{equation}
    where $\lambda_{h}^{+}$ is the smallest nonzero eigenvalue of $M^\star_h$. Moreover, all the eigenvalues are upper bounded by:
    \begin{equation}
        k +\lambda - f(k).
    \end{equation}
    From Proposition~\ref{prop:kanto}:
    \begin{align}
        \norm{\phi_h(s,a)}_{(\Lambda_h^k)^{-1}} &\le \sqrt{x^\transp (\Lambda_h^k)^{-1}x} \\
        &\le \sqrt{x^\transp (B_h^k)^{-1}x} \\
        &\le \frac{k+\lambda-f(k)}{k\lambda_{h}^{+} +\lambda - f(k)}\frac{1}{\sqrt{x^\transp B_h^kx}}.\label{pp:key.1}
    \end{align}
    Again from the properties of $B_h^k$, $x$ is orthogonal to all the orthonormal eigenvector of $B_h^k$ that correspond to zero eigenvalues of $\Lambda_h^\star$. Hence by Proposition~\ref{prop:minposeig}:
    \begin{align}
        x^\transp B_h^k x &= kx^\transp \Lambda_h^\star x + \lambda - f(k) \\
        &\ge k \min_{y\in\Imm(\Lambda_h^\star),\norm{y}=1} y^\transp \Lambda_h^\star y +\lambda - f(k) \\
        &= k \lambda_{h}^{+} +\lambda - f(k).\label{pp:key.2}
    \end{align}
    Since $\beta_k = \wt{O}(1)$ and $f(k) = \wt{O}(\sqrt{k})$, from~\eqref{pp:key.1} and~\eqref{pp:key.2}:
    \begin{equation}
        \beta_k\norm{\phi_h(s,a)}_{(\Lambda_h^k)^{-1}} \le \beta_k\frac{k+\lambda-f(k)}{(k\lambda_{h}^{+}+\lambda - f(k))^{3/2}} = \wt{O}(k^{-1/2}).
    \end{equation}
\end{proof}

\begin{lemma}\label{lemma:abcd}
Let $\{\phi_j\}_{j\in[n]}$ be a set of $n$ vectors in $\mathbb{R}^d$ and $v\in\mathbb{R}^d$ be such that $v \notin  \mathrm{span}\{\phi_j : j\in[n]\} $. Then, there exists a scalar $\epsilon > 0$ such that, for any $t \geq 0, \eta > 0$, 
\begin{align*}
\|v\|_{(t\sum_{j\in[n]}\phi_j\phi_j^T + \eta I)^{-1}} \geq \frac{\epsilon}{\sqrt{\eta}}.
\end{align*}
\end{lemma}
\begin{proof}
Let $\{\lambda_i, u_i\}_{i\in[d]}$ denote the eigenvalues/eigenvectors of the matrix $\sum_{j\in[n]}\phi_j\phi_j^T$. Note that $\mathrm{span}\{u_i : i\in[d]\} = \mathrm{span}\{\phi_j : j\in[n]\}  \subset \mathbb{R}^d$. Then, Lemma 28 of \cite{papini2021leveraging} ensures that there exists a scalar $\epsilon > 0$ such that $|v^T u_i| \geq \epsilon$ for at least one eigenvector $u_i$ associated with a zero eigenvalue. Noting that the eigenvectors of $(t\sum_{j\in[n]}\phi_j\phi_j^T + \eta I)^{-1}$ are the same as the those of $\sum_{j\in[n]}\phi_j\phi_j^T$, we have that
\begin{align*}
\|v\|_{(t\sum_{j\in[n]}\phi_j\phi_j^T + \eta I)^{-1}}^2 = \sum_{j\in[d]} \frac{(v^T u_j)^2}{\eta + \lambda_j} \geq \frac{(v^T u_i)^2}{\eta} \geq \frac{\epsilon^2}{\eta},
\end{align*}
which concludes the proof.
\end{proof}

\section{Examples and Numerical Validations}\label{app:examples.experiments}
Consider the following two-stage MDP ($H=2$) with states $\Sspace=\{s_1,s_2\}$ and actions $\Aspace=\{a_1,a_2\}$:
\begin{align}
    &r_1(s,a) = 1 &&\text{for all $s\in\Sspace$ and $a\in\Aspace$}, \\
    &p_1(s_1|s_1,a_1) = 1, &&p_1(s_1|s_1,a_2)=\frac{1}{2}, &&p_1(s_1|s_2,a_1)=\frac{1}{2},
    &&p_1(s_1|s_2,a_2)=\frac{3}{4}, \\
    &r_2(s_1,a_1) = 1,
    &&r_2(s_1,a_2)=\frac{7}{8},
    &&r_2(s_2,a_1)=\frac{1}{2},
    &&r_2(s_2,a_2)=\frac{5}{8},
\end{align}
$\mu(s_1)=\mu(s_2)=1/2$, and of course
$p(s_2|s,a)=1-p(s_1|s,a)$ for all $s\in\Sspace$ and $a\in\Aspace$.
Backward induction shows that the (unique) optimal policy is:
\begin{align}
    &\pi^\star_1(s_1) = a_1, &&\pi^\star_1(s_2) = a_2, 
    &&\pi^\star_2(s_1) = a_1,
    &&\pi^\star_2(s_2) = a_2,
\end{align}
with the following values:
\begin{align}
    &V^\star_1(s_1) = 2,
    &&V^\star_1(s_2) = \frac{61}{32},
    &&V^\star_2(s_1) = 1,
    &&V^\star_2(s_2) = \frac{5}{8}.
\end{align}
Notice also that all states and actions are reachable, i.e. $\rho_h(s,a)>0$ for all $s\in\Sspace$, $a\in\Aspace$, and $h\in[H]$.

\paragraph{\hls{} representation.}
Consider the following $2$-dimensional representation $\Phi^{(1)}$:
\begin{align}
    &\underline{\phi_1^{(1)}(s_1,a_1)} = \begin{bmatrix}1\\0\end{bmatrix} 
    &&\phi_1^{(1)}(s_1,a_2) = \begin{bmatrix}1/2\\1/2\end{bmatrix}
    &&\phi_1^{(1)}(s_2,a_1) = \begin{bmatrix}1/2\\1/2\end{bmatrix}
    &&\underline{\phi_1^{(1)}(s_2,a_2)} = \begin{bmatrix}3/4\\1/4\end{bmatrix} \\
    &\underline{\phi_2^{(1)}(s_1,a_1)} = \begin{bmatrix}0\\1\end{bmatrix} 
    &&\phi_2^{(1)}(s_1,a_2) = \begin{bmatrix}1/4\\3/4\end{bmatrix}
    &&\phi_2^{(1)}(s_2,a_1) = \begin{bmatrix}1\\0\end{bmatrix}
    &&\underline{\phi_2^{(1)}(s_2,a_2)} = \begin{bmatrix}3/4\\1/4\end{bmatrix}.
\end{align}

It is easy to check that the MDP is low-rank (Asm~\ref{asm:lowrank}) and $\Phi^{(1)}$ is a realizable representation with $\vtheta_1=[1,1]^\transp$, $\boldsymbol{\mu}_1(s_1)=[1,0]^\transp$, $\boldsymbol{\mu}_1(s_2)=[0,1]^\transp$, and $\vtheta_2=[1/2, 1]^\transp$. This is an example of low-rank MDP with \emph{simplex feature space} (see Example 2.2 in~\citep{Jin2020linear}).
We have underlined optimal features. It is easy to see that optimal features span $\Reals^2$ at both stages\footnote{It may appear counterintuitive that simplex features, which live on a one-dimensional manifold, can span $\Reals^2$. However, notice that the simplex is not a \emph{linear} subspace of the Euclidean space (it does not include the origin). Indeed, we could describe the example MDP with less parameters, but we would loose the linear structure.}, so $\Phi^{(1)}$ is \hls{}. The optimal covariance matrices are:
\begin{align}
    &\Lambda_{1,\star}^{(1)} = \frac{1}{32}\begin{bmatrix}25 &3\\3&1\end{bmatrix},
    &&\Lambda_{2,\star}^{(1)}=\frac{1}{128}\begin{bmatrix}9 &3\\3&113\end{bmatrix}.
\end{align}
Both are full rank, and their minimum eigenvalues are:
\begin{align}
    &\lambda_{1,+}^{(1)} = \frac{13-3\sqrt{17}}{32} \simeq 0.02,
    &&\lambda_{2,+}^{(1)} = \frac{61-\sqrt{2713}}{128} \simeq 0.07.
\end{align}
As shown in Theorems~\ref{thm:const.ele} and~\ref{thm:const.lsvi}, both LSVI-UCB and ELEANOR will only suffer constant regret on this problem.

\paragraph{Non-\hls{} representation.}
We apply the procedure described in the proof of Lemma 7 from~\citep{papini2021leveraging} to the second stage
 of $\Phi^{(1)}$ to obtain an equivalent representation $\Phi^{(2)}$:
\begin{align}
    &\underline{\phi_2^{(2)}(s_1,a_1)} = \begin{bmatrix}30/89\\74/89\end{bmatrix} 
    &&\phi_2^{(2)}(s_1,a_2) = \begin{bmatrix}1/4\\3/4\end{bmatrix}
    &&\phi_2^{(2)}(s_2,a_1) = \begin{bmatrix}1\\0\end{bmatrix}
    &&\underline{\phi_2^{(2)}(s_2,a_2)} = \begin{bmatrix}75/356\\185/356\end{bmatrix},
\end{align}
while the feature map for $h=1$ is the same.
It is easy to check that this is still a realizable representation for our MDP with the same parameters.\footnote{However, notice that some of the new features do not belong to the simplex.} 
Although the \hls{} property holds for $h=1$, it no longer does for $h=2$. Indeed, we have the following linear dependence between optimal features:
\begin{equation}
    \phi_{2,\star}^{(2)}(s_2) = \frac{5}{8} \phi_{2,\star}^{(2)}(s_1),
\end{equation}
so optimal features only span $\Reals^1$. However, suboptimal features still span $\Reals^2$, e.g., by taking action $a_2$ in $s_1$ and $a_1$ in $s_2$ (recall that all state-action pairs are reachable).
Due to Theorem~\ref{thm:necessity}, neither LSVI-UCB nor ELEANOR will achieve constant regret on this problem.

\paragraph{Alternative Non-\hls{} representation.}
It is also easy to build a representation that is non-\hls{} by changing the representation at the first stage.
For example, let $h$ be any stage (e.g., $h=1$ in our example) for which we want to transform a \hls{} representation (in our case $\phi^{(1)}$) into a non-\hls{} one. We can define a new representation $\phi^{(3)}$ as follows

\begin{align}
    \label{eq:tr1}
    s \in\Sspace, ~~\phi^{(3)}_h(s,a^\star) = \begin{bmatrix}0_d\\\phi_h^{(1)}(s,a^\star_s)\end{bmatrix}
    && \forall a \neq a^\star_s, ~~\phi^{(3)}_h(s,a) = \begin{bmatrix}\phi_h^{(1)}(s,a)\\0_d \end{bmatrix}\\
    \label{eq:tr2}
    s' \in\Sspace,~~\mu^{(3)}_h(s’) = \begin{bmatrix}\mu_h^{(1)}(s’)\\ \mu_h^{(1)}(s’)\end{bmatrix}
    && \theta^{(3)}_h = \begin{bmatrix}\theta_h^{(1)}\\ \theta_h^{(1)} \end{bmatrix}
\end{align}
Since all states are reachable, it is easy to verify that $\lambda_{\min} \Big(\mathbb{E}_{s \sim \rho^\star_h} \Big[ \phi^{(3)}_h(s,a^\star_s)^\intercal \phi^{(3)}_h(s,a^\star_s) \Big] \Big) = 0$ and that 
\[
\mathrm{span}\Big\{\phi_h(s,a) \;|\; \forall (s,a), \;\exists \pi \in \Pi : \rho^\pi_h(s,a) > 0\Big\} \neq \mathrm{span}\Big\{\phi_h^\star(s) \;|\; \forall s,\; \rho^{\star}_h(s) > 0\Big\}.
\]
Then, the representation is not \hls{} at stage $h$.

\subsection{Numerical Validations}
We provide a numerical validation of the behavior of the algorithms with and without a \hls{} representation.
We consider the following representations: $\phi^{(1)}$, $\phi^{(2)}$, $\phi^{(3)}$ which is obtained by applying the transformation in Eq.~\ref{eq:tr1}-\ref{eq:tr2} to $\phi^{(2)}$ at stage $h=1$, and $\phi^{(4)}$ which is obtained by applying the transformation in Eq.~\ref{eq:tr1}-\ref{eq:tr2} to $\phi^{(1)}$ at stage $h=1$. Note that we have $d_1 = 4$ and $d_2 =2$ for $\phi^{(3)}$ and $\phi^{(4)}$. Furthermore, $\lambda_{1,1}^{(3)}=\lambda_{1,1}^{(4)} = 0$, while $\lambda_{2,1}^{(3)}=0$ and $\lambda_{2,1}^{(4)} > 0$, which means that $\phi^{(4)}$ is ``locally'' \hls{} at stage $h=2$. The reward is stochastic and drawn from a Bernoulli distribution: $r_{h,t} \sim \text{Ber}(r_h(s_t,a_t))$.
We tested both LSVI-UCB on each individual representation and LSVI-LEADER with different combinations of the representations. 
We consider $\beta_{h,k} = c_{\beta} d_h H \sqrt{\log(d_h K)}$ and $\beta_{h,k} = c_{\beta} d_h H\sqrt{N\log(N d_h K)}$ for LSVI-UCB and LSVI-LEADER, respectively. We set $c_{\beta}=0.2$ and $K = 30000$.
The regret is shown in Fig.~\ref{fig:exp_1}, averaged over the same $100$ seeds.

As expected from the theoretical analysis, LSVI-UCB with \hls{} representation  suffers constant regret since, after the initial exploration phase, it only selects optimal actions. On the other hand, when the representation is Non-\hls{}, LSVI-UCB suffers a non-constant regret that grows over episodes. 
LSVI-LEADER is able to exploit the structure of the \hls{} representation and it achieves constant regret as well in all the configurations containing a \hls{} representation. The higher regret is due to a longer exploration phase that is a consequence of the enlarged confidence intervals; this is also in line with the theoretical analysis.
It is interesting to notice that LSVI-LEADER performs equally good with all the combinations of representations of dimension three (i.e., $\{\phi^{(1)},\phi^{(2)},\phi^{(4)}\}$, $\{\phi^{(1)},\phi^{(2)},\phi^{(3)}\}$ and $\{\phi^{(2)},\phi^{(3)},\phi^{(4)}\}$). LSVI-LEADER is indeed able to mix representations and achieve constant regret even when none of the individual representation would.  In the case of $\{\phi^{(2)},\phi^{(3)},\phi^{(4)}\}$, LSVI-LEADER is able to mix $\phi^{(2)}$ and $\phi^{(4)}$, that are \hls{} in stage $h=1$ and $h=2$, respectively.

\begin{figure}[t]
    \includegraphics[width=.48\textwidth]{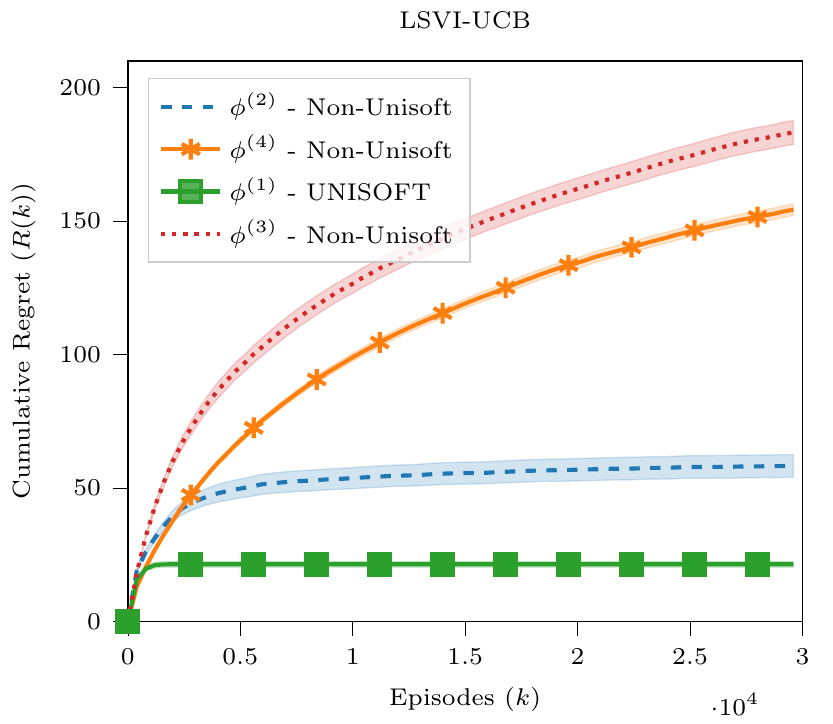}
    \hfill
    \includegraphics[width=.48\textwidth]{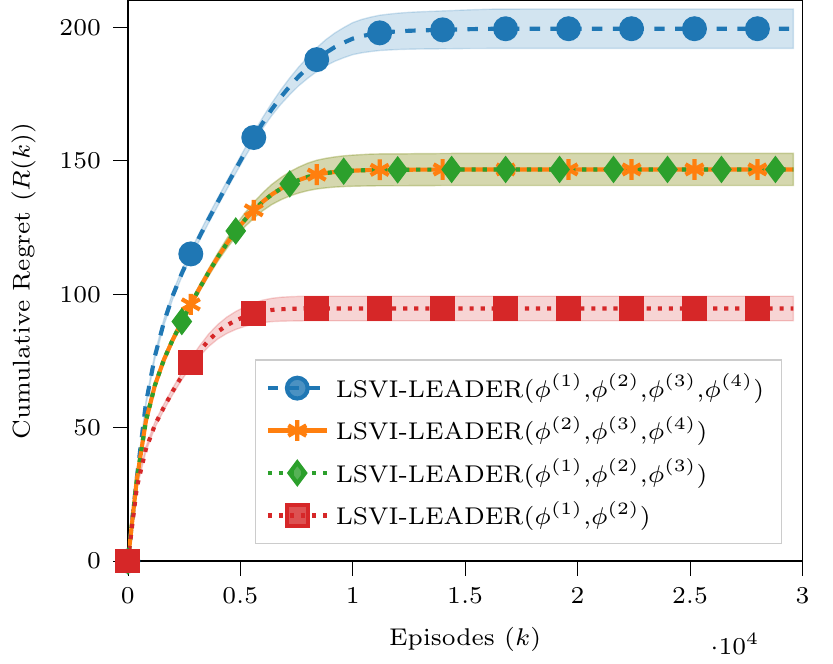}
    \caption{Cumulative regret of LSVI-UCB and LSVI-LEADER with different representations. The performance of LSVI-LEADER with $\{\phi^{(1)},\phi^{(2)},\phi^{(4)}\}$ is the same of the one with $\{\phi^{(1)},\phi^{(2)},\phi^{(3)}\}$ and $\{\phi^{(2)},\phi^{(3)},\phi^{(4)}\}$.}
    \label{fig:exp_1}
\end{figure}

\paragraph{\hls{} in DeepRL.}
We wanted also to verify the existence of \hls{} representations in DeepRL. We trained A2C~\citep{MnihBMGLHSK16} on different domains and evaluated whether the recovered representation (i.e., last layer of the neural network used to approximate $V^\star$) satisfies the \hls{} assumptions. Standard benchmark problems are not finite-horizon, we thus considered the following ``strong'' \hls{} condition  $\lambda_{\min}\left(\mathbb{E}_{s \sim \rho^\star} [\phi^\star(s) \phi^\star(s)^\transp]\right) >0$, which was evaluated by simulating multiple trajectories: 
\begin{equation}
    \Lambda_m^\pi = \frac{1}{m} \sum_{i=1}^m \sum_{t=1}^{T_i} \phi(s_t,a_t) \phi(s_t,a_t)^\transp
\end{equation}
where $a_t = \pi(s_t)$. We use a deterministic version of the policy recovered by A2C for evaluation. We trained A2C using the implementation provided by stable-baselines3~\citep{stable-baselines3}. 
We use the default parameters (provided by rl-baselines3-zoo~\citep{rl-zoo3}
) and tested different network architectures. Since we did not optimize the parameter, we reported only the domains where we obtained good results with at least one network architecture (highlighted in the table).
Since A2C estimates directly $V^\star$, we used the features of the last layer as features of the optimal policy (i.e., $\phi^\star(s)$) to test for the ``strong'' \hls{} condition.

Tables~\ref{tab:deep16}--\ref{tab:deep64} show that in several domains the learnt representation is \hls{}, although the minimum eigenvalue is small. As expected, the number of ``strong'' \hls{} representations decreases as the size of the last layer increases. This initial experiment shows that \hls{} representations are not uncommon in practice but also leave open the possibility of designing algorithms that explicitly try to force the \hls{} while learning. We believe this is an interesting direction for future work.

\begin{table}
    \centering
    \scriptsize
    \begin{tabular}{|l|c|c|c|c|c|c|c|}
    \hline
        Domain & \makecell{mean\\reward} & \makecell{std\\reward} & \makecell{eval\\timesteps} & \makecell{eval\\episodes ($m$)} & $\mathrm{rank}(\Lambda_m^\pi)$ & $\lambda_{\min}(\Lambda_m^\pi)$ & \hls{} \\ \hline
        \rowcolor{LightCyan}
        Acrobot-v1 & -84.5 & 20.7 & 149923 & 1753 & 16 & 0.02 & \faCheck \\ \hline
        \rowcolor{LightCyan}
        AntBulletEnv-v0 & 2303.9 & 68.3 & 150000 & 150 & 15 & 0 & \\ \hline
        BipedalWalker-v3 & 2.2 & 1.6 & 148800 & 93 & 10 & 0 & \\ \hline
        \rowcolor{LightCyan}
        CartPole-v1 & 500.0 & 0.0 & 150000 & 300 & 1 & 0 & \\ \hline
        HopperBulletEnv-v0 & 836.3 & 536.2 & 149982 & 372 & 16 & 0 & \\ \hline
        \rowcolor{LightCyan}
        MountainCar-v0 & -124.9 & 31.4 & 149979 & 1201 & 16 & 0.01 & \faCheck \\ \hline
        \rowcolor{LightCyan}
        MountainCarContinuous-v0 & 91.6 & 0.2 & 149966 & 1736 & 5 & 0 & \\ \hline
        \rowcolor{LightCyan}
        Pendulum-v0 & -173.5 & 107.0 & 150000 & 750 & 16 & 0 & \\ \hline
    \end{tabular}
    \caption{Results for A2C policy network of dimension $[64,16]$ and value network of dimension $[64, 16]$. These dimensions represent the size of the hidden layers (with tanh activation function). We highlighted the environments where A2C achieved good performance.}
    \label{tab:deep16}
\end{table}

\begin{table}
    \centering
    \scriptsize
    \begin{tabular}{|l|c|c|c|c|c|c|c|}
    \hline
        Domain & \makecell{mean\\reward} & \makecell{std\\reward} & \makecell{eval\\timesteps} & \makecell{eval\\episodes ($m$)} & $\mathrm{rank}(\Lambda_m^\pi)$ & $\lambda_{\min}(\Lambda_m^\pi)$ & \hls{} \\ \hline
        \rowcolor{LightCyan}
        Acrobot-v1 & -84.9 & 29.4 & 149987 & 1747 & 32 & 0.0018 & \faCheck \\ \hline
        \rowcolor{LightCyan}
        AntBulletEnv-v0 & 2109.9 & 46.1 & 150000 & 150 & 32 & 0.0010 & \faCheck \\ \hline
        \rowcolor{LightCyan}
        BipedalWalker-v3 & 267.3 & 53.3 & 149278 & 201 & 24 & 0 &  \\ \hline
        \rowcolor{LightCyan}
        CartPole-v1 & 500.0 & 0.0 & 150000 & 300 & 1 & 0 &  \\ \hline
        \rowcolor{LightCyan}
        HopperBulletEnv-v0 & 1461.6 & 707.1 & 149123 & 205 & 32 & 0.0001 & \faCheck \\ \hline
        \rowcolor{LightCyan}
        MountainCar-v0 & -116.5 & 28.0 & 149999 & 1288 & 32 & 0.0001 & \faCheck \\ \hline
        \rowcolor{LightCyan}
        MountainCarContinuous-v0 & 91.5 & 0.2 & 149975 & 1742 & 10 & 0 &  \\ \hline
        \rowcolor{LightCyan}
        Pendulum-v0 & -236.5 & 187.7 & 150000 & 750 & 26 & 0 & \\ \hline
    \end{tabular}
    \caption{Results for A2C policy network of dimension $[64,32]$ and value network of dimension $[64, 32]$. These dimensions represent the size of the hidden layers (with tanh activation function). We highlighted the environments where A2C achieved good performance.}
    \label{tab:deep32}
\end{table}

\begin{table}
    \centering
    \scriptsize
    \begin{tabular}{|l|c|c|c|c|c|c|c|}
    \hline
        Domain & \makecell{mean\\reward} & \makecell{std\\reward} & \makecell{eval\\timesteps} & \makecell{eval\\episodes ($m$)} & $\mathrm{rank}(\Lambda_m^\pi)$ & $\lambda_{\min}(\Lambda_m^\pi)$ & \hls{} \\ \hline
        \rowcolor{LightCyan}
        Acrobot-v1 & -83.3 & 17.1 & 149970 & 1778 & 64 & 0.0003 & \faCheck \\ \hline
        \rowcolor{LightCyan}
        AntBulletEnv-v0 & 1912.7 & 106.0 & 150000 & 150 & 64 & 0.0008 & \faCheck \\ \hline
        \rowcolor{LightCyan}
        BipedalWalker-v3 & 276.1 & 25.8 & 149707 & 198 & 28 & 0 & \\ \hline
        \rowcolor{LightCyan}
        CartPole-v1 & 500.0 & 0.0 & 150000 & 300 & 2 & 0 & \\ \hline
        HopperBulletEnv-v0 & 14.0 & 0.8 & 149997 & 26620 & 59 & 0 & \\ \hline
        \rowcolor{LightCyan}
        MountainCar-v0 & -107.3 & 20.1 & 149944 & 1397 & 44 & 0 & \\ \hline
        \rowcolor{LightCyan}
        MountainCarContinuous-v0 & 92.4 & 0.1 & 149984 & 1948 & 10 & 0 & \\ \hline
        \rowcolor{LightCyan}
        Pendulum-v0 & -153.3 & 92.9 & 150000 & 750 & 32 & 0 & \\ \hline
    \end{tabular}
    \caption{Results for A2C policy network of dimension $[64,64]$ and value network of dimension $[64, 64]$. These dimensions represent the size of the hidden layers (with tanh activation function). We highlighted the environments where A2C achieved good performance.}
    \label{tab:deep64}
\end{table}

\end{appendix}
\end{document}